\documentclass{article}

\usepackage[numbers]{natbib}

\usepackage[final]{neurips_2022}

\usepackage{diagbox}
\usepackage{url}
\usepackage{graphics}
\usepackage{graphicx}
\usepackage{multirow}
\usepackage{subfigure}
\usepackage{xcolor}
\usepackage{booktabs}
\usepackage{wrapfig}
\usepackage{sidecap}
\usepackage{etoc}
\usepackage{amsmath}
\usepackage{amsfonts}
\usepackage{bm}

\newcommand{\cL}{{\mathcal L}}

\usepackage{pifont}

\definecolor{mygray}{gray}{0.4}

\usepackage{Definitions}
\usepackage{comment}
\usepackage{floatrow}
\newfloatcommand{capbtabbox}{table}[][\FBwidth]
\usepackage{blindtext}
\DeclareMathOperator{\id}{id}

\DeclareMathOperator{\sg}{sg}

\DeclareMathOperator*{\argmin}{arg\,min}

\let\ab\allowbreak

\newcommand{\modelname}{{\textsc{DGRL}}}

\usepackage{algorithm}
\usepackage{algorithmic}

\usepackage{listings}
\usepackage{xcolor}

\definecolor{codegreen}{rgb}{0,0.6,0}
\definecolor{codegray}{rgb}{0.5,0.5,0.5}
\definecolor{codepurple}{rgb}{0.58,0,0.82}
\definecolor{backcolour}{rgb}{0.95,0.95,0.92}

\lstdefinestyle{mystyle}{
    backgroundcolor=\color{backcolour},   
    commentstyle=\color{codegreen},
    keywordstyle=\color{magenta},
    numberstyle=\tiny\color{codegray},
    stringstyle=\color{codepurple},
    basicstyle=\ttfamily\footnotesize,
    breakatwhitespace=false,         
    breaklines=true,                 
    captionpos=b,                    
    keepspaces=true,                 
    numbers=left,                    
    numbersep=5pt,                  
    showspaces=false,                
    showstringspaces=false,
    showtabs=false,                  
    tabsize=2
}

\title{Discrete Factorial Representations as an \\
Abstraction for Goal Conditioned RL}

\author{\parbox{\linewidth}{Riashat Islam$^{1,6,7}$\thanks{Corresponding Author E-mails: riashat.islam@mail.mcgill.ca, airudhgoyal9119@gmail.com}, Hongyu Zang$^3$, Anirudh Goyal$^{2, 5}$,
Alex Lamb$^{7}$, 
Kenji Kawaguchi$^{4}$,\\ Xin Li$^{3}$, Romain Laroche$^{7}$,
Yoshua Bengio$^{2}$,  Remi Tachet Des Combes$^{7}$ } \vspace{0.3em} \\ 
$^1$ McGill University, Mila, Quebec AI Institute 
$^2$ University of Montreal, Mila, Quebec AI Institute\\
$^3$ Beijing Institute of Technology 
$^4$ Harvard University 
$^5$ DeepMind \\
$^6$ Microsoft Research, New York
$^7$ Microsoft Research, Montreal\\
}

\begin{document}
\maketitle
\vspace{-1.4em}
\begin{abstract}
Goal-conditioned reinforcement learning (RL) is a promising direction for training agents that are capable of solving multiple tasks and reach a diverse set of objectives.  How to \textit{specify} and \textit{ground} these goals in such a way that we can both reliably reach goals during training as well as generalize to new goals during evaluation remains an open area of research. Defining goals in the space of noisy and high-dimensional sensory inputs poses a challenge for training goal-conditioned agents, or even for generalization to novel goals. We propose to address this by learning factorial representations of goals and processing the resulting representation via a discretization bottleneck, for coarser goal specification, through an approach we call DGRL. We show that applying a discretizing bottleneck can improve performance in goal-conditioned RL setups, by experimentally evaluating this method on tasks ranging from maze environments to complex robotic navigation and manipulation. Additionally, we prove a theorem lower-bounding the expected return on out-of-distribution goals, while still allowing for specifying goals with expressive combinatorial structure.
\end{abstract}



\section{Introduction}
\vspace{-1mm}
Reinforcement Learning is a popular and highly general framework \citep{kaelbling1996reinforcement, sutton2018reinforcement} focusing on how to select actions for an agent to yield high long-term sum of rewards. An important question is how to control the desired behavior of an RL agent, both during training and evaluation \citep{kaelbling1987architecture}.  One way to control this behavior is by specifying a reward signal \citep{schultz1998predictive, silver2021reward}.  While this approach is very general, the reward signal can be hard to design and may not be the most informative form of feedback.  The credit assignment problem in RL can become difficult when the reward signal is sparse \citep{sutton1984temporal, tesauro1991practical, mnih2013playing, moore1993prioritized, singh1996reinforcement}, such as policy gradients becoming nearly flat in regions where reward is almost never achieved.  Generalization can also suffer if the agent only learns one way to achieve a high reward rather than learning a diverse set of skills for coping with novel challenges \citep{haarnoja2018soft}.

One potential way to flexibly specify and ground the desired behavior of RL agents is by training agents that receive a reward when they reach a goal specified explicitly to them \citep{Kaelbling93}. In this approach, called \textit{Goal-Conditioned RL}, a single agent is trained to reach a diverse set of goals, and is given a reward only when it reaches the goal it was instructed to reach \citep{sutton2011horde, schaul2015universal, pong2018temporal}.  This provides a richer signal for the agent than simply collecting more samples oriented around a single goal, as reaching multiple goals requires the agent to learn a more diverse and robust set of skills.  It also allows for more flexible and tightly constrained control over its desired behavior \citep{dayan1992feudal, barto2003recent, kulkarni2016hierarchical}. Finally, the diversity of goals seen during training should help improve both credit assignment and generalization \citep{schaul2015universal, plappert2018multi}.  

While this framework is promising, it introduces two new challenges: \textit{goal grounding} \citep{chao2011towards, akakzia2020grounding} and \textit{goal specification} \citep{bahdanau2018learning}.  Goal grounding refers to defining the goal space, and goal specification refers to selecting what goal the agent should try to reach in a given context.  The agent is only rewarded in goal-conditioned RL when reaching the reward it was instructed to reach, whereas in goal-free RL a reward is provided regardless of any such specification, which makes the nature of the agent's task fundamentally different.

\begin{wrapfigure}{rt}{0.53\textwidth}
  \begin{center}
    \includegraphics[width=\linewidth]{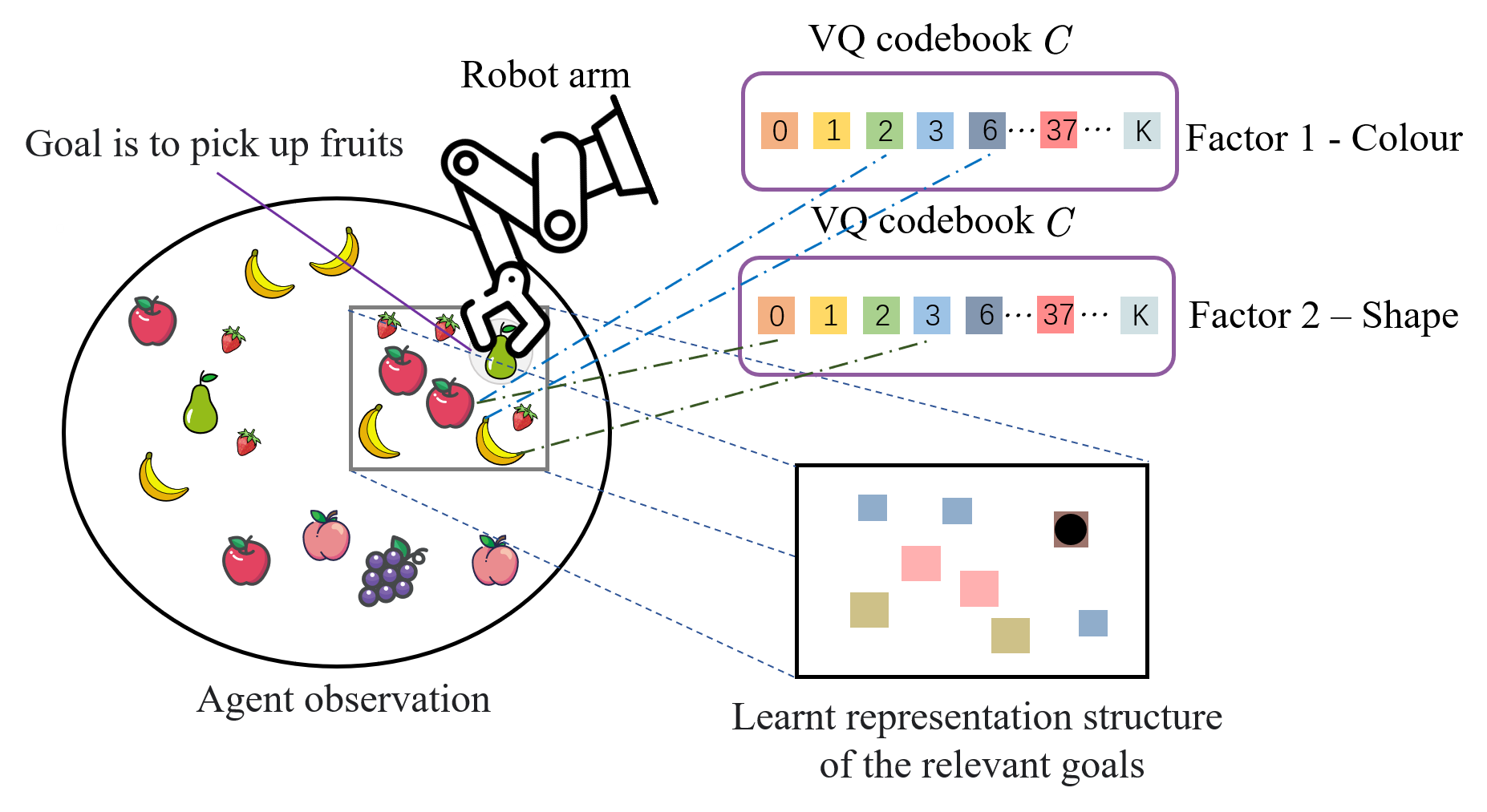}
  \end{center}
  \vspace{-1.5mm}
  \caption{ \small Illustration of learning \textit{discrete} and \textit{factorial} goal representations.}
  \label{fig:illustration}
  \vspace{-3mm}
\end{wrapfigure}

What makes grounding and specifying goals challenging? Consider trying to train a goal-conditioned RL agent to pick up various fruits from a table. For example, we may want it to pick up a red apple or a green pear (illustrated in Figure~\ref{fig:illustration}). The number of possible goals of interests may be fairly small, such as the set of all valid combinations of fruits and their colors, while the number of possible observations of goals is extremely large when working in a rich observation space (\textit{e.g} images from a camera).  \textit{Goal Grounding} refers to this challenge of relating high-dimensional observations and the space of relevant goals. \textit{Goal Specification} refers to picking a suitable goal for the agent to reach and computing an appropriate reward when it is reached. It also implies specifying goals reachable in the agent's current context \citep{nachum2018near, khetarpal2020can}. Goal specification can be done either manually by a developer or by another RL agent, such as a high-level agent which generates goals a lower-level agent then tries to reach \citep{dayan1992feudal, barto2003recent, kulkarni2016hierarchical, jiang2019language}.  Goals specified in language are an excellent fit for these desiderata, as language is a compressed discrete representation which is useful for out-of-distribution generalization, while being compositional and expressive \citep{grice1975logic, dal2002planning, jiang2019language, fu2019language,Adhikari2020,Weir2022}.  At the same time, connecting language feedback for an agent is non-trivial, requiring special assumptions or a labeling framework \citep{chevalier2018babyai}. 

We propose to learn the goal representations with self-supervised learning (either trained on their own, or jointly with the downstream RL objective) while forcing them to be \textit{discrete} and \textit{factorial}. To perform this discretization, we use Vector-Quantization \citep{van2017neural, razavi2019generating, liu2021discrete} which discretizes a continuous representation using a codebook of discrete and learnable codes. The approach proposed here, called \modelname, serves two complementary purposes. First, it provides a structured representation of the raw visual goals. By representing the visual goals as a \textit{composition} of discrete codes from a learned dictionary, it simplifies the grounding of unseen goals, \textit{i.e.}, goals not seen during training, to novel compositions of the trained discrete codes. We show empirically that this improves the generalization performance of goal-reaching policies while remaining expressive enough. Second, the learned discrete codes can be used by another agent (like a higher-level policy in hierarchical RL) to specify sub-goals to a lower-level policy, and eventually complete the task (\textit{i.e.}, reach the final goal). In this case, goal-inference is learned end-to-end.  The effectiveness of goal-conditioned HRL relies on the specification of semantically meaningful sub-goals. Using factorial discrete sub-goals allows the higher-level policy to specify semantically meaningful objectives to the lower-level policy. 



\vspace{-1.5em}
\section{Preliminaries}
\vspace{-1.5mm}
\textbf{Goal-conditioned RL.} We consider a goal-conditioned Markov Decision Process, where the goals $g\in\mathcal{G}$ live in the state space $\mathcal{S}$, \textit{i.e.}, $\mathcal{G} = \mathcal{S}$. We denote a goal-conditioned policy as $\pi(a | s, g)$ (either stochastic or deterministic), and its expected total return as $ J(\pi) =   \mathbb{E} \Big[ \sum_{t=0}^{T} R(s_t, g, a)   \Big]$ where the goal $g$ is either sampled from a distribution $\rho_g$ or provided by another higher level policy $\pi_{\theta_{h}}^{h}(g | s)$. The value function $V^{\pi}$ is additionally conditioned on goals, and is trained to predict the expected sum of future rewards conditioned on states and goals; $V^{\pi}(s,g) = \mathbb{E} \Big[  \sum_{t=0}^{T} R(s_t, g, a) \mid s_0 = s;\pi      \Big]$. As in standard RL, the objective in goal-conditioned RL is to maximize the expected discounted returns induced by the goal-conditioned policy. 

\textbf{Hierarchical Reinforcement Learning.} We consider goal-conditioned settings in which the goals are specified in the observation space.  In the hierarchical reinforcement learning (HRL) setup,  goals are provided by a higher level policy $\pi_{\theta_{h}}^{h}(g | s_{t})$. The higher level policy operates at a coarser time scale and chooses a goal $g_{t} \sim \pi_{\theta_{h}}^{h}(g | s_{t}) $ to reach for the lower level policy every K steps. The lower level policy executes primitive actions $\pi_{\theta_{l}}^{l}(a | s_{t}, g_{t})$ to reach the goals specified by the high-level policy and is trained to maximize the intrinsic reward provided by the high-level policy. The  higher level policy is trained to maximize the external reward \textit{i.e.}, the reward function specified by the MDP. Both the higher and lower level policies can be trained with any standard RL algorithms, such as Deep Q-Learning (DQN) \citep{mnih2013playing} or policy optimization based algorithms \citep{schulman2015trust,schulman2017proximal,Laroche2021}.  Alternately, one can also consider another setup for goal-conditioned RL, where the goals are provided by the environment $g_1, \dots g_L$, and are part of the state or observation space. At each episode of training, one of the goals is sampled from the distribution of goals $\rho_g$ and the policy is trained to reach the sampled goal. At test time, the agent can be evaluated either on its ability to reach goals within the distribution $\rho_g$, or for its out-of-distribution generalization capability to reach new kinds of goals. We consider both the HRL and goal-conditioned setups, and evaluate the significance of learning factorial representation of discrete latent goals in a series of complex goal-conditioned tasks.

\textbf{Vector Quantized Representations.} VQ-VAE \cite{van2017neural, razavi2019generating, liu2021discrete} discretizes the bottleneck representation of an auto-encoder by adding a codebook of discrete learnable codes. The input is passed through an encoder. The output of the encoder is compared to all the vectors in the codebook, and the codebook vector closest to the continuous encoded representation is fed to the decoder. The decoder is then tasked with reconstructing the input from this quantized vector. 

\textbf{Self-supervised learning of representations.} Several papers \citep{laskin2020curl, stooke2021decoupling, schwarzer2020data, mazoure2020deep} have demonstrated the benefits of using a pre-training stage where the representations of raw states are learned using self-supervised objectives in a task-agnostic fashion. After the pre-training stage, the representations can be used for (and potentially also fine-tuned on) downstream tasks. These self-supervised representations have been shown to improve sample efficiency.  

\section{Discrete Goal-Conditioned Reinforcement Learning (DGRL)}
\vspace{-1mm}
\begin{figure}[t]
    \centering
    \includegraphics[width=0.8\linewidth]{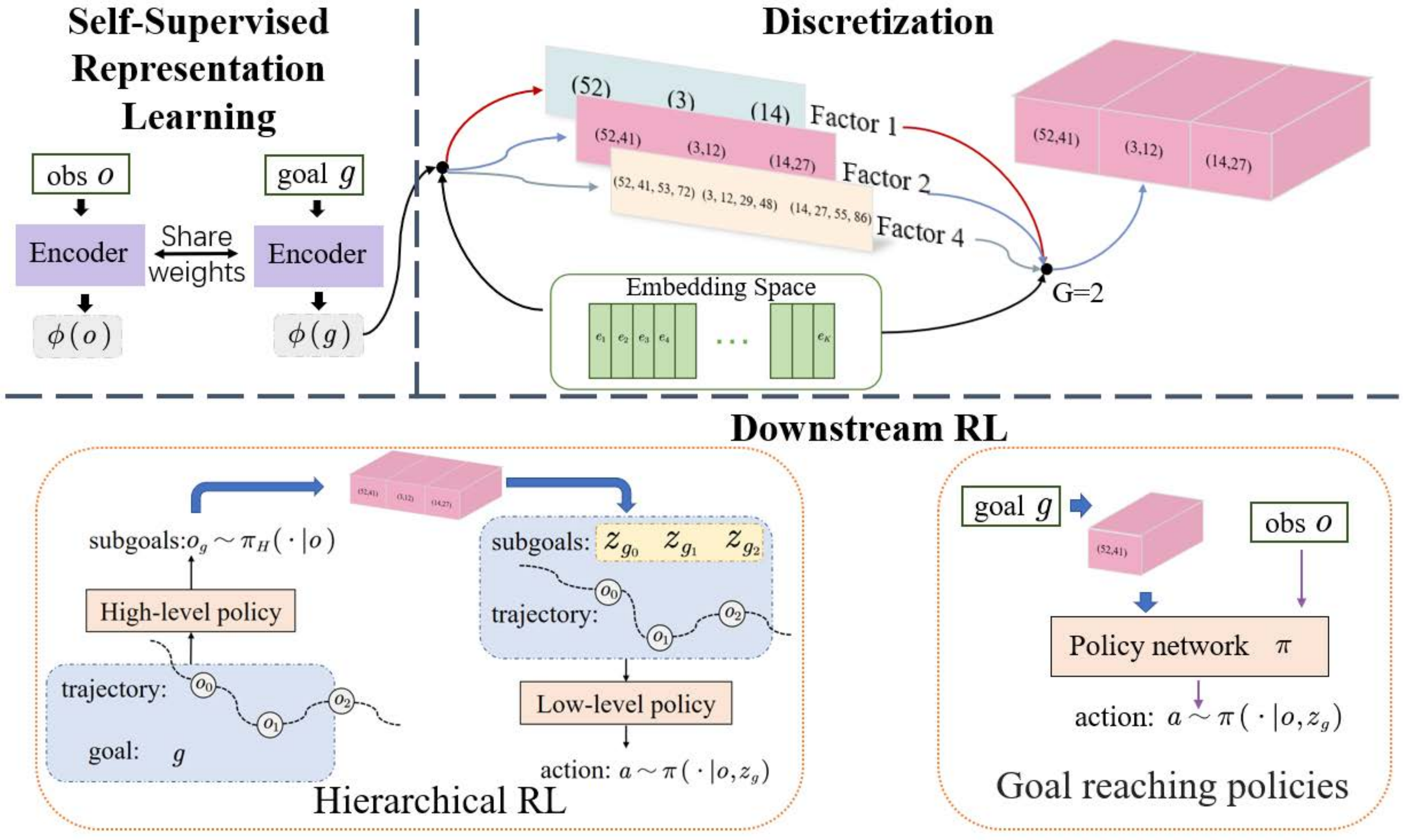}
    \label{fig:archit}%
    \caption{ \textbf{Summary of Proposed $\modelname$ Model} for improving \textit{goal grounding} and \textit{goal specification} by making goal representations \textit{discrete} and \textit{factorial}.  We learn a latent representation  for both observations and goals using a self-supervised learning method (sec. \ref{sec:ssl_learning}). We convert the learnt latent representation into discrete latents based on a VQ-VAE quantization bottleneck with multiple factor outputs  (sec. \ref{sec:learned_codebook}). We use the resulting discrete representations for downstream RL tasks: (i) to train a goal-conditioned policy or value function, and (ii) in the context of goal-conditioned hierarchical reinforcement learning (sec. \ref{sec:downstream_rl}). \vspace{-1.4em}
  }
\end{figure}
In this section, we provide technical details on the proposed framework, $\modelname$,  which consists of three parts: (a) learning representations of raw visual observations through self-supervised representation objectives, (b) processing the resulting representations via a learned dictionary of discrete codes, and (c) using the resulting discrete representations for downstream goal-conditioned and HRL tasks. We later describe, in Section \ref{sec:experiments}, how discrete goal representations can accelerate learning in complex navigation and manipulation tasks.  We emphasize that these representations can be learned at the same time as the downstream-RL objective or pre-trained with self-supervised learning, and then used as a fixed representation for RL.  


\paragraph{Self-Supervised Goal Representation Learning. } \label{sec:ssl_learning}
One can use any off-the shelf self-supervised method for learning representations of the raw state and the goal observations. We denote by $\phi$ the encoder network that takes as input the raw state and maps it to a continuous embedding: $z_e = \phi(s)$. Here, we explore two different self-supervised techniques for learning representations.  For simpler environments, we use a simple autoencoder with its standard reconstruction objective.  For more complex environments, we use the Deep InfoMax approach \citep{mazoure2020deep} which optimizes for a contrastive objective as a proxy to maximizing the mutual information between representations of nearby states in the same trajectory.

\paragraph{Processing continuous representations via a discrete codebook.} \label{sec:learned_codebook}
We learn discrete representations by using the vector-quantization method from the VQ-VAE paper \cite{van2017neural}, and follow the multi-factor setup used in Discrete-Value Neural Communication \cite{liu2021discrete}.  The discretization process for each vector $z_{e}  \in \Hcal\subset \RR^{m}$ is described as follows. First, vector $z_{e}$ is divided into $G$ segments $c_1, c_2,\dots, c_G$, where $z_{e}=\textsc{concatenate}(c_1, c_2, \dots, c_G),$ and each segment $c_i \in \RR^{m/G}$ (implying that $m$ is divisible by $G$). Each continuous segment $c_i$ is mapped independently to a discretized latent  vector $e \in \RR^{L \times (m/G)}$ where $L$ is the size of the discrete latent space (\textit{i.e.}, an $L$-way categorical variable):
$$
e_{o_i}=\textsc{discretize}(c_i), \quad \text{ where } o_i=\argmin_{j \in \{1,\dots,L\}} ||c_{i}-e_j||.
$$
These discrete codes, which we call the factors of the continuous representation $z_{e}$, are concatenated to obtain the final discretized vector $z_{q}$: 
\begin{equation}
\label{discrete_codes}
\begin{aligned}
z_{q}=\textsc{concatenate}(\textsc{discretize}(c_1), \textsc{discretize}(c_2),...,\textsc{discretize}(c_G)).
\end{aligned}
\end{equation}
The loss for vector quantization is: $\mathcal{L}_{\mathrm{discretization}} = \frac{\beta}{G}  \sum^{G}_{i}||c_{i}-\sg(e_{o_i})||^2_2$.

The training procedure closely follows both \cite{liu2021discrete} and \cite{van2017neural}.  Here, $\sg$ refers to a stop-gradient operation that blocks gradients from flowing into $e_{o_i}$, and $\beta$ is a hyperparameter which controls how strongly we move the codes toward the encoded values.  Unlike \cite{liu2021discrete}, we used a moving average to update the code embeddings rather than learning them directly as parameters. We update $e_{o_i}$ with an exponential moving average to encourage it to become close to the selected output segment $c_i$.  This update sets the new value of $e_{o_i}$ to be equal to $\eta e_{o_i} + (1-\eta) c_i$, where the value of $\eta$ is a fixed hyperparameter controlling how quickly the moving average updates.\footnote{Note that this could also be thought of as a gradient step on $e_{o_i}$ taken in the direction $c_i - e_{o_i}$.} The term $\sum^{G}_{i}||c_{i}-\sg(e_{o_i})||^2_2$ is often called the commitment loss.  We trained the VQ-quantization process together with other parts of the model by gradient descent. When there were multiple $z_e$ vectors to discretize in a model, the mean of the commitment loss across all $z_e$ vectors was used. 

\textbf{Summary.} The multiple steps described above can be summarized by $z_{q}=q(z_{e},L,G)$,  where $L$ is the codebook size, $G$ the number of factors per vector, and $q(\cdot)$ the whole discretization process.  We train the representations for both the state and goal observations with this discretization bottleneck applied to the  continuous representations resulting from the self-supervised training. The number of  factors $G$ is a hyper-parameter. In our experiments, we explored different values: $G=1, 2, 4, 8, 16$, and found that $G=16$ worked the best.  Discretizing with more factors slightly increases computation but reduces the number of model parameters due to the codebook embeddings being reused across the different factors.

\subsection{Using representations for downstream RL}
\label{sec:downstream_rl}

We use the discrete representations for downstream RL tasks: (i) to train a goal-conditioned policy, and (ii) in the context of hierarchical reinforcement learning. 

\textbf{Goal-conditioned RL.} Defining goals in the space of noisy, high-dimensional sensory inputs poses a challenge for generalization to novel goals because the encoder that maps the goal observations to the low dimensional latent representation may fail to generalize. One way to address this is to embed the continuous latent representation into a discrete representation such that the representation of the novel goal is mapped to the fixed set of latent discrete codes. This facilitate generalization to new combinations of these codes while making it easy for downstream learning to figure out the meaning of each discrete code. In this setup, instead of feeding the continuous state and goal embeddings to the agent, we use their discretized versions, thus grounding goal representations in the input space. 

We use the resulting representations to train a goal-conditioned policy $a_t \sim \pi_{\theta_l}^{l}(a| s_t, g_t)$ or a goal-conditioned action value function $Q(s_t, a_t, g_t)$. At each training episode, a goal is sampled from the goal distribution $\rho_g$, and the agent gets rewarded for reaching it. This reward can either be \textit{extrinsic}, \textit{i.e.}, part of the environment, or \textit{intrinsic}, \textit{i.e.}, part of the algorithm.  In $\modelname$, we define the intrinsic reward as the fraction of discrete factors which match in the respective representations of the goal observation and of the state observation.  At test time, the agent can either be evaluated on reaching goals within the distribution $\rho_g$, or for its generalization capability to goals not seen during training.


\textbf{Hierarchical RL.} The higher level policy $g_{t} \sim \pi_{\theta_{h}}^{h}(g \mid s_{t})$ outputs a continuous representation of goals $g$ by conditioning on the states every $K$ time-steps, it can also  output a sub-goal $s_g$ by conditioning on both states $s$ and environment goals $g$,\textit{ i.e.}, $\pi_{\theta_{h}}^{h}(s_g \mid s_{t}, g_{t})$. The effectiveness of goal-conditioned HRL relies on the specification of semantically meaningful sub-goals. Learned codebooks (Section \ref{sec:learned_codebook}) consisting of a set of discrete codes can be used by a higher level policy to \textit{specify} which goal to reach to a lower level policy. The use of learned codebooks ensures that the goal specified by the higher level policy is grounded in the space of raw-observations. 

In Section \ref{sec:experiments}, we empirically show the benefits of the proposed approach for training goal-reaching policies or goal-conditioned value functions, as well as in a goal-conditioned hierarchical RL setup. 

\vspace{-3mm}
\section{Theoretical Analysis}
\vspace{-2mm}
\label{sec:theory_analysis}
In this section, discretization is shown to improve generalization to novel goals by enhancing the concentration of the goal distribution within each neighborhood of discretized goal values; \textit{i.e.}, by decomposing the goal probability $p(g)$ into $p(g)=\sum_k p(g|g\in \Gcal_k)p(g \in \Gcal_k)$ with the neighborhood set $\{\Gcal_k\}_k$, it improves the overall performance in $p(g)$ by increasing the concentration in $p(g|g\in \Gcal_k)$. Intuitively, this is because the discretization removes varieties of possible goal values $g\in \Gcal_k$ for each neighborhood $\Gcal_k$.  To state our result, we define $\varphi_\theta(g)=\EE_{s_0}[V^{\pi}(s_0,g)]$, where $\theta \in \RR^m$ is the vector containing model parameters learned through $n$ goals observed during training phase, $g_1,\dots,g_n$. We denote the discretization of $g$ by $q(g)$, and the identity function by $\id$ as $\id(g)=g$. Let $\Qcal=\{q(g) : g \in \Gcal \}$ and $\hat d$ be a distance function. We use $\Qcal_i$ to denote the $i$-th element of $\Qcal$ (by ordering elements of $\Qcal$ with an arbitrary ordering). We also define $[n]=\{1,\dots,n\}$,  $
\Gcal_k=\{g\in \Gcal  :k= \argmin_{i \in [|\Qcal|]}  \hat d(q(g),\Qcal_i) \}
$,
 $
\Ical_{k}^{}=\{i\in[n]: g_i\in \Gcal_{k}\}, 
$ and $\Ical_Q=\{k \in [|\Qcal|]:|\Ical_{k}| \ge 1 \}$. We denote by $c$ a constant in  $(n,\theta,\Theta,\delta,S)$. 

The following theorem (proof in Appendix~\ref{sec:app_theory}) shows that the goal discretization improves the lower bound of  the expected sum of rewards for unseen goals $\EE_{g \sim \rho_g}[ (\varphi_\theta \circ \varsigma)(g))]$ by the margin of $\omega(\theta)$: 
\begin{theorem} \label{thm:1}
For any $\delta>0$,  with  probability at least $1-\delta$, the following holds for any $\theta \in \RR^m$ and $\varsigma \in \{\id, q\}$: 
\begin{align*}
\EE_{g \sim \rho_g}[ (\varphi_\theta \circ \varsigma)(g))]  \ge \frac{1}{n} \sum_{i=1}^n  (\varphi_\theta \circ \varsigma)(g_{i})-  c  \sqrt{\frac{2\ln(2/\delta)}{n}} - \one\{\varsigma=\id\}\omega(\theta)
\end{align*} 
where $\omega(\theta)=\frac{1}{n}\sum_{k\in \Ical_Q}|\Ical_{k}^{}|\left(\frac{1}{|\Ical_{k}|}\sum_{i \in \Ical_{k}}\varphi_\theta(g_{i}) - \EE_{g \sim \rho_g}[\varphi_\theta (g)|g\in  \Gcal_{k}] \right)$. Moreover,  for any compact  $\Theta \subset \RR^m$, if   $\varphi_\theta(g)$ is continuous at each $\theta \in \Theta$ for almost all $g$ and is dominated by a function  $\chi$ as  $|\varphi_\theta(g)|\le \chi(g)$ for all $\theta \in \Theta$ with  $\EE_{g}[\chi(g)]<\infty$,  then the following holds:  
$$
\sup_{\theta \in \Theta}|\omega(\theta)| {\xrightarrow  {P}}\ 0 \ \ \ \text{ when } \ \ \ n \rightarrow \infty. 
$$
\end{theorem}
\begin{proof}
Detailed proof provided in the Appendix \ref{sec:app_theory}
\end{proof}

Without the goal discretization, we incur an extra cost of $\omega(\theta)$, which is expected to be strictly positive since $\frac{1}{|\Ical_{k}|}\sum_{i \in \Ical_{k}}\varphi_\theta(g_{i})$ is maximized during training while $ \EE_{g \sim \rho_g}[\varphi_\theta (g)|g\in  \Gcal_{k}]$ is not.
Thus, the goal discretization can improve the expected sum of rewards for unseen goals by the degree of $\omega(\theta)$, which measures the concentration of the goal distribution in each neighborhood. This extra cost $\omega(\theta)$ goes to zero when the number of goal observations $n$ approaches infinity. 

\section{Related Work}
\label{ref:related_work}
Learning with multiple hierarchies has long been proposed in the RL literature, where goal conditioned HRL implements high level planning and low level control using sub-goals. Often in goal conditioned HRL, the higher level policy specifies goals which may not have good \textit{specification} and \textit{grounding}. Several prior works focus on goal-conditioned RL to improve sample efficiency in deep RL tasks \cite{NachumGoal,NachumGoal2}. These build on ideas that were proposed years back to solve long horizon tasks by hierarchical RL specifying goals \cite{Kaelbling93,DayanH92,Dietterich00,WieringS97}. The goal is to learn to solve sub-goals provided to the policy, by learning to predict a sequence of actions that can reach each of the sub-goals \cite{Veeriah_ManyGoals,SchaulHGS15,NairF20}. Additionally, in existing HRL literature, distance measures are often used based on goal-conditioned value functions, allowing to measure distances between states and the sequence of sub-goals to reach \cite{EysenbachSL21, Zhang0S21a}, for planning \cite{LEAP}, or exploration~\cite{Laroche2017}. Since the set of goals specified in the state space can be arbitrary, an additional constraint is often also learnt to tie the distribution of selected goals to those the lower level policy can reach \cite{HRAC}. We tackle this problem by proposing $\modelname$, for better \textit{grounding} and \textit{specification} of sub-goal representations.

In previous works, mutual information based objectives have been proposed for goal conditioned RL. They perform goal-based representation learning, in order to improve stability of training goal-conditioned value functions \cite{NachumGoal, NachumGoal2}, or to provide goal representations allowing to identify decision states for better exploration \cite{GoyalISALBBL19}. However, for most of these settings, the sub-goals are based on an external reward and are lacking in terms of specification, which can lead to inefficient training. \cite{LevyKPS19, HRAC} have proposed approaches that penalize the high level controller for generating sub-goals that are too difficult for the lower level policies, through the use of additional constrained objectives \cite{HRAC}. We highlight that our proposed $\modelname$ can be generically applied to any goal conditioned RL literature for better \textit{grounding} and \textit{specification} of the sub-goals. Furthermore, recent work has shown significance of learning representations through self supervised objectives in RL \cite{AnandROBCH19, SchwarzerAGHCB21}, often as a pre-training phase \cite{SchwarzerRNACHB21, YangN21}, which can help for both exploration \cite{MisraHK020} and control \cite{YaratsFLP21}.

In the context of goal conditioned RL, it can be a challenging problem as it additionally requires learning reliable representations of goals in parallel, purely from high dimensional observations \cite{Dwiel}. Previous works \cite{LevyKPS19} have often used the entire observation space as goals, which is not scalable for complex tasks. Other works have used a pre-defined space of sub-goals as domain knowledge \cite{NachumGoal}, or self-play for sub-goal representations \cite{Sukhbaatar} to reduce the complexity of goal space design. Most recently, \cite{LEAP, NairF20} utilized unsupervised representation learning to learn a goal representation space, which can further be used for planning and control. In this work, we show additionally that using a bottleneck can further lead to factorial representation of goals, while helping with goal specification via learning a latent space of discrete goals usable for planning and control. We emphasize that $\modelname$ can be integrated on any existing goal conditioned approach that utilizes learning a sub-goal representation.

\section{Experiments}
\label{sec:experiments}
The main goal of our experiments is to show that goal discretization can lead to sample efficient learning and generalization to novel goals, in goal-conditioned RL. 
First, we directly study this by training on environments with a set of goals (such as 8 positions within a gridworld) and then evaluating the agent's ability to reach a position within the gridworld which it was not trained to reach.  Second, we consider hierarchical goal-conditioned RL, in which a higher-level agent generates goals that a lower-level agent is tasked with reaching.  In this case, the task of reaching novel goals occurs \textit{organically} as the higher-level model selects new goals.  This setup also shows the advantages of $\modelname$ for \textit{goal specification}.  A secondary goal of our experiments is to show that using many discrete factors is often critical for optimal performance, which proves the value of \textit{factorization} in \textit{grounding goals}.

We evaluate our proposed method $\modelname$ by integrating it into existing state-of-the-art goal-conditioned and hierarchical RL tasks. Experimentally, we analyse $\modelname$ on several challenging testbeds that have previously been used in the RL community. $\modelname$ in principle can be applied to any existing downstream goal-conditioned RL tasks. We demonstrate improvements on five such tasks.  We consider maze navigation where images are used as observations and we show improved generalization to novel goals. We integrate $\modelname$ to an existing goal-conditioned baseline for navigating procedurally-generated hard exploration Minigrid environments \cite{gym_minigrid} and find that it outperforms state-of-the-art exploration baselines. We also show improvements with $\modelname$ on continuous control (Ant) navigation and manipulation tasks, where goals come from a high-level controller.  Finally, we show that discrete representations also significantly improve sample efficient learning on a challenging vision-based robotic manipulation environment. 

\textbf{Demonstrating Factorized Representation Learning} We first pick a color-mnist supervised learning example to support the idea that $\modelname$ can learn factorized or compositional representations. 
\begin{figure}[!htbp]
\centering
\subfigure[Original Image]{
\includegraphics[
width=0.22\textwidth]{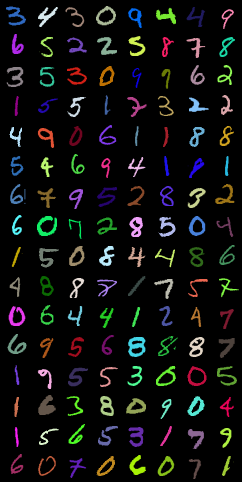}
}
\subfigure[Both factors recon.]{
\includegraphics[
width=0.22\textwidth]{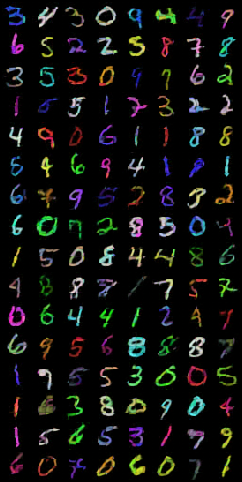}
}
\subfigure[Factor 1 recon.]{
\includegraphics[
width=0.22\textwidth]{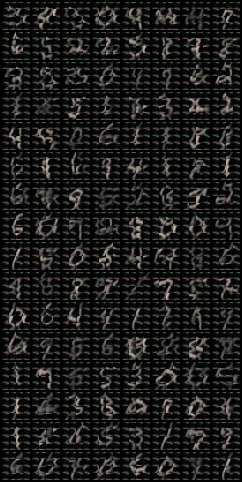}
}
\subfigure[Factor 2 recon.]{
\includegraphics[
width=0.22\textwidth]{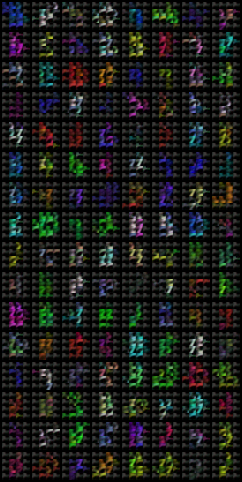}
}
\caption{Color-MNIST example to demonstrate factorized representations; reconstructing the original images with two factors. Leftmost : Original Image; Left-Middle : Reconstructed Image without substitution; Right-Middle : Reconstructed Image with one groups of discrete codes substituted by zero vectors; Rightmost : Reconstructed Image with the other groups of discrete codes substituted by zero vectors}
\label{fig:color_mnist_compositionality}
\end{figure}
Figure~\ref{fig:color_mnist_compositionality} displays reconstructed images from a trained decoder operating on a discretized 2-factor representation. We find that different factors capture information of different semantic nature. More precisely, factor 1 tends to encode the shape of the digit, while factor 2 specialized in its color. This empirically suggests the emergence of ``factorization'' in the learnt representations. Further experimental details are provided in section \ref{app:rebuttal_details_mnist}.

\textbf{Learning to Reach Diverse and Novel Goals.} We study a gridworld navigation task in which an agent is trained to reach a goal from a small finite set of training goals, and during evaluation is tasked with reaching a novel goal unseen during training.  This is a navigation task with a pixel-level observation space showing the position of the agent and the goal in a gridworld. We consider two mazes spiral and single-loop topology. Experiment setup is given in Appendix~\ref{sec:app_vis_maze_details}. 

For this task, we train a goal-conditioned Deep Q-Learning (DQN) agent, and use a pre-trained representation $\phi(\cdot)$ where the encoder is trained using data from a random rollout policy. Because the gridworld is small the random rollout policy achieves good coverage of the state space, so we found this was sufficient for learning a good goal representation.  At each episode, a specific goal is randomly sampled from a distribution of goals, and the DQN agent is trained to reach the specified goal for that episode. During evaluation, we test the learned agent on goals either from the training distribution, or not seen during training.

Furthermore, for this task, we additionally use an intrinsic reward to promote exploration of the goal-DQN agent. Since we learn a discrete factorial representation of the goal, we compute an exploration bonus based on the discrete latent codebooks; \textit{i.e.}, we embed the states and goals using the learned codes and then compute an intrinsic exploration bonus based on the fraction of learned factors that match. For the baseline goal-DQN agent, we provide an additional reward bonus based on the cosine distance between continuous embeddings of the state observation and goal. Figure \ref{fig:4_vis_maze} shows that $\modelname$ significantly outperforms a continuous baseline goal DQN agent, when trained on either four goals or eight goals. We evaluate generalization to 4 novel goals unseen during training (Figure \ref{fig:8_ood_vis_maze}) and demonstrate improved generalization.

\begin{figure}[!htbp]
\centering
\subfigure{
\includegraphics[
width=0.35\textwidth]{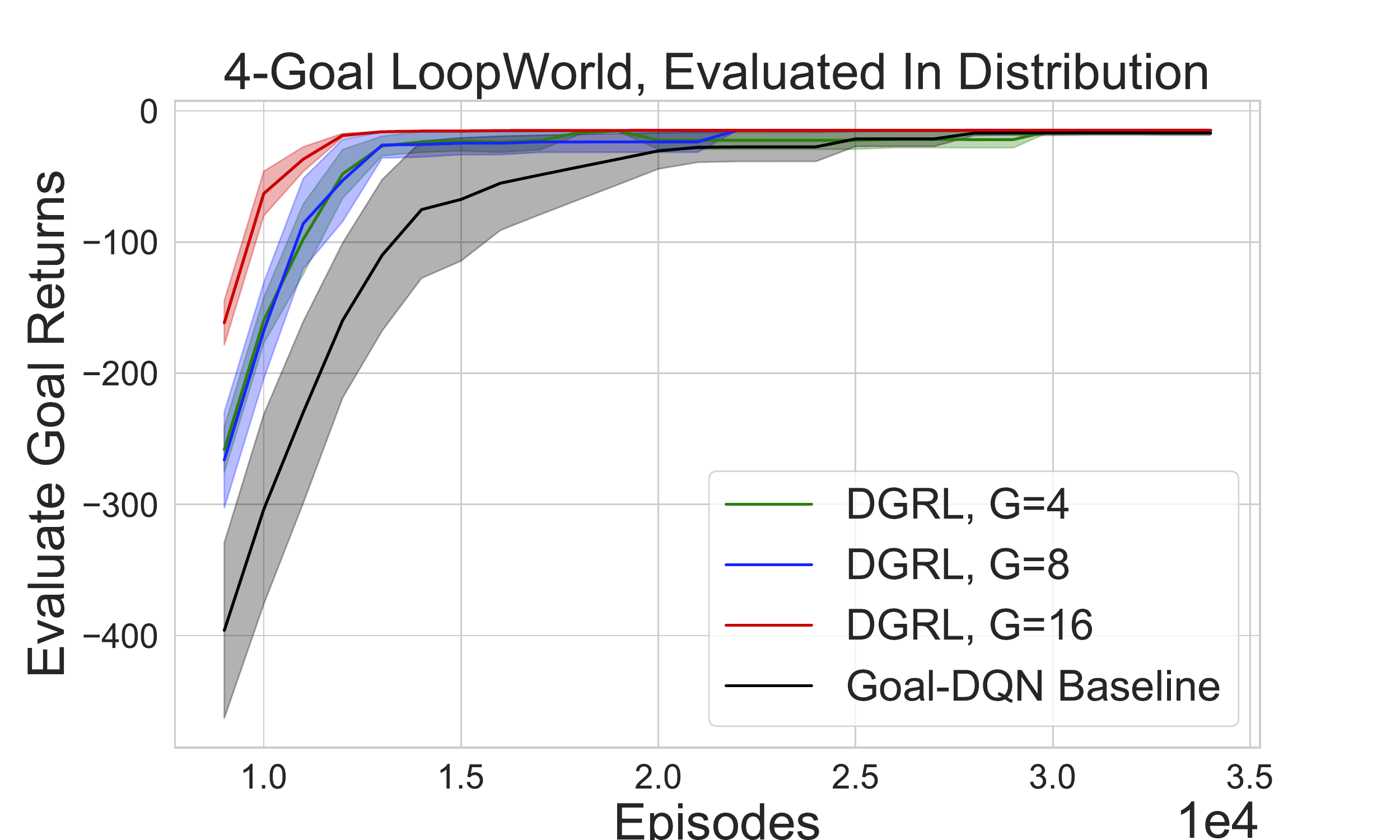}
}
\hspace{-0.8cm}
\subfigure{
\includegraphics[
width=0.35\textwidth]{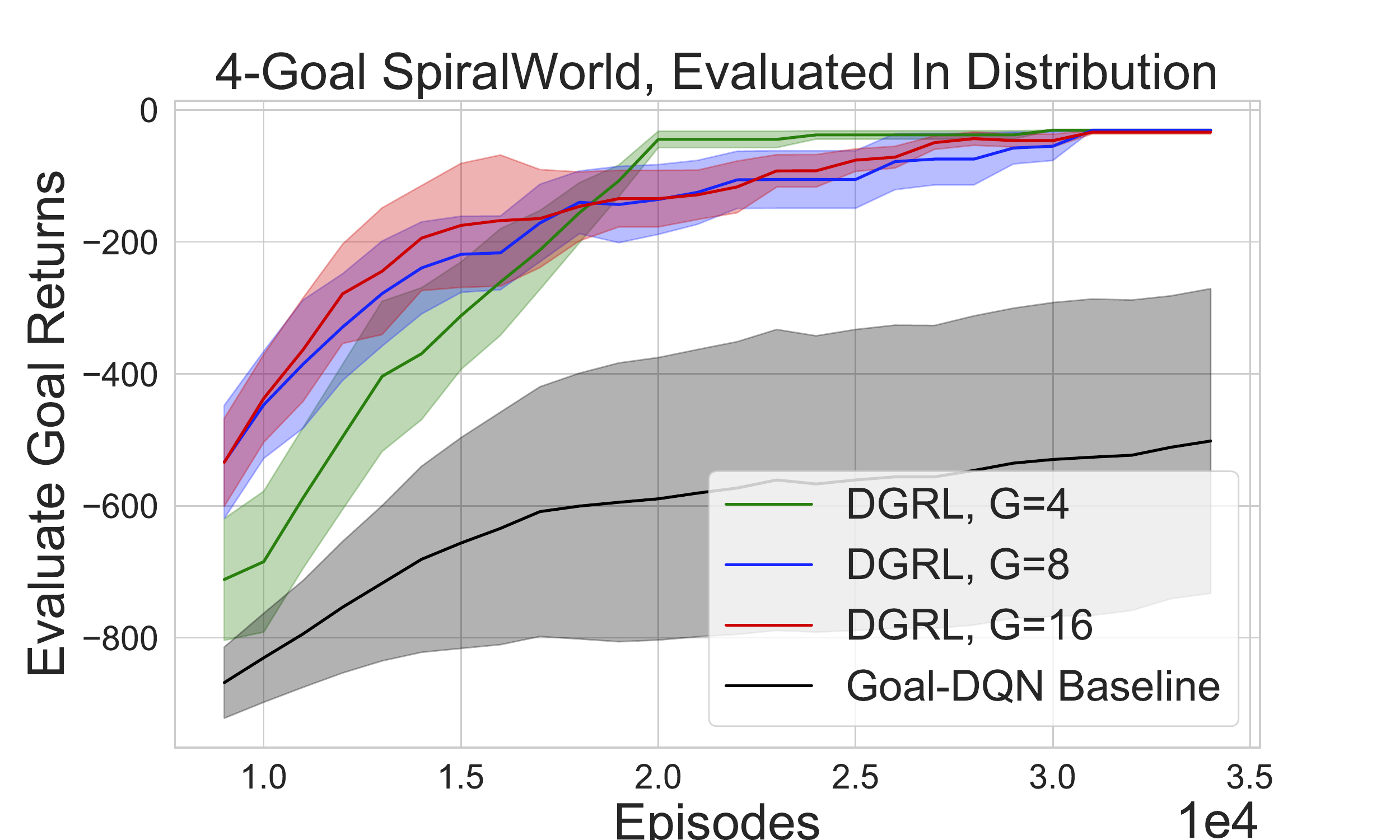}
}
\hspace{-0.8cm}
\subfigure{
\includegraphics[
width=0.35\textwidth]{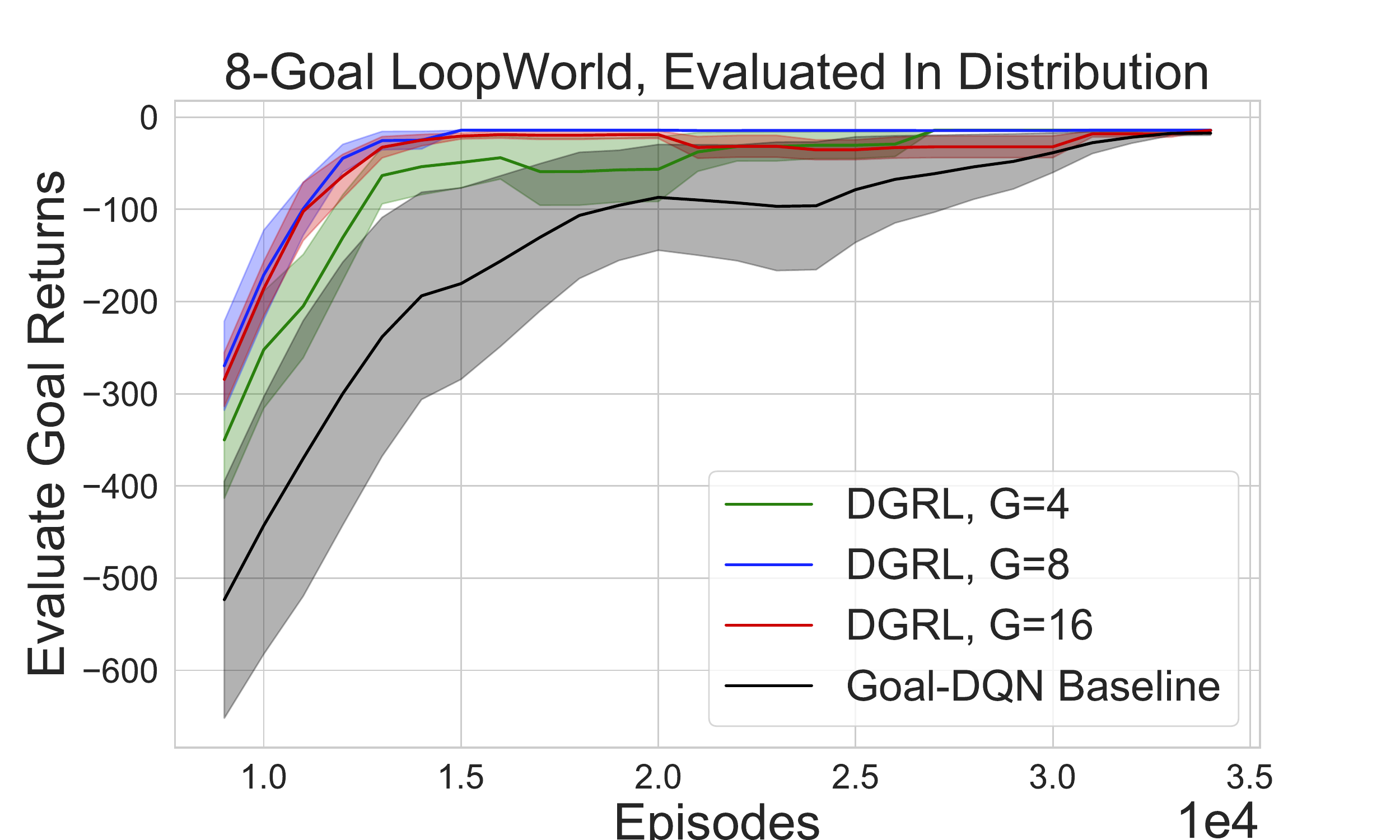}
}
\caption{Loopworld maze environment. We show that for $4, 8$ and $16$ different discrete factors, $\modelname$ outperforms a goal-DQN baseline agent with continuous goal representations. As we increase the number of factors $G$ to 16, the expressivity of the discrete goal representation increases, lowering the odds of the factors being the same. This provides a better intrinsic reward signal for exploration, resulting in faster convergence for $\modelname$ integrated on a goal-DQN agent.}
\label{fig:4_vis_maze}
\end{figure}

\begin{figure}[!htbp]
    \centering
    {{\includegraphics[width=3.2cm]{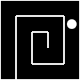} }}%
    \hspace{-0.2cm}
    {{\includegraphics[width=5.6cm]{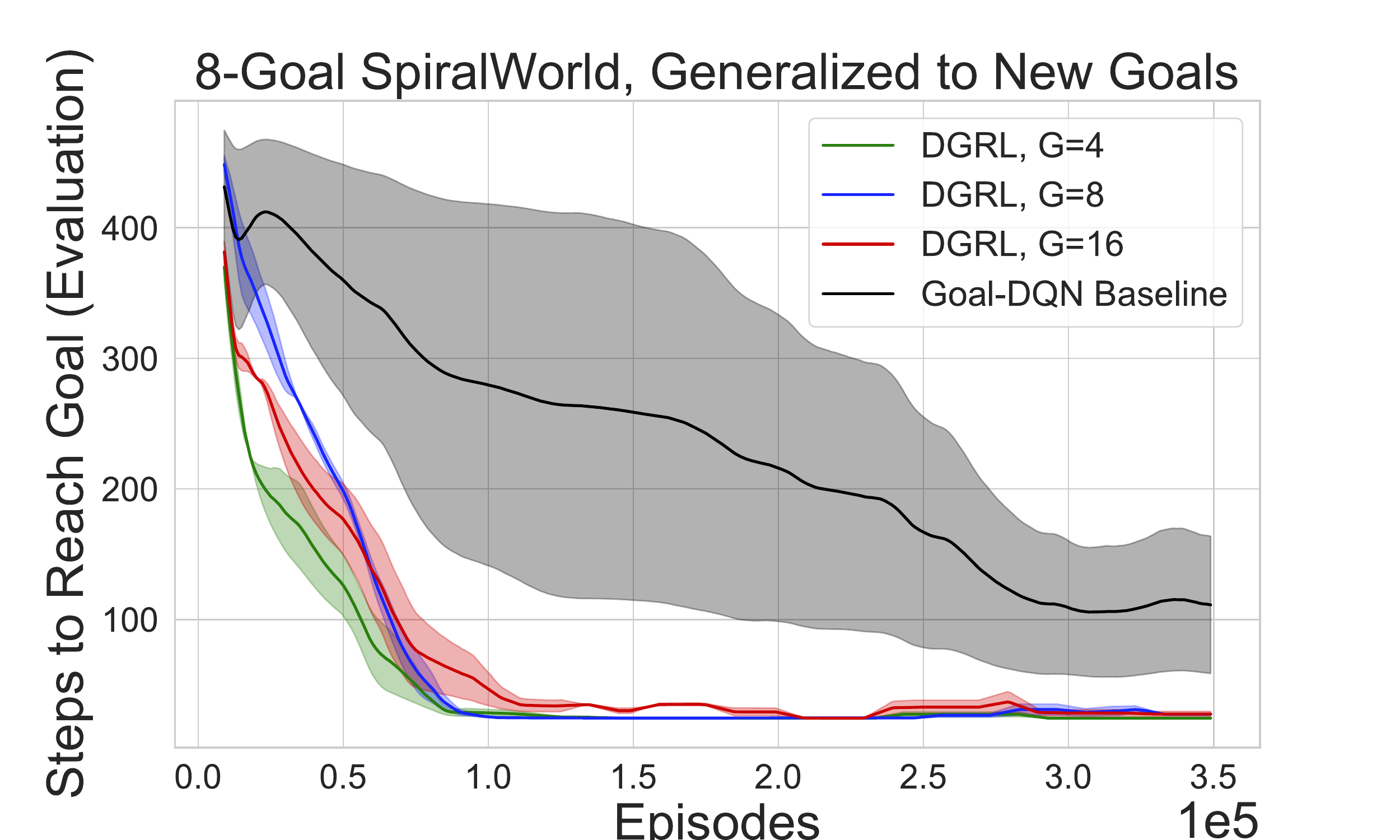} }}%
    \hspace{-0.7cm}
    {{\includegraphics[width=5.6cm]{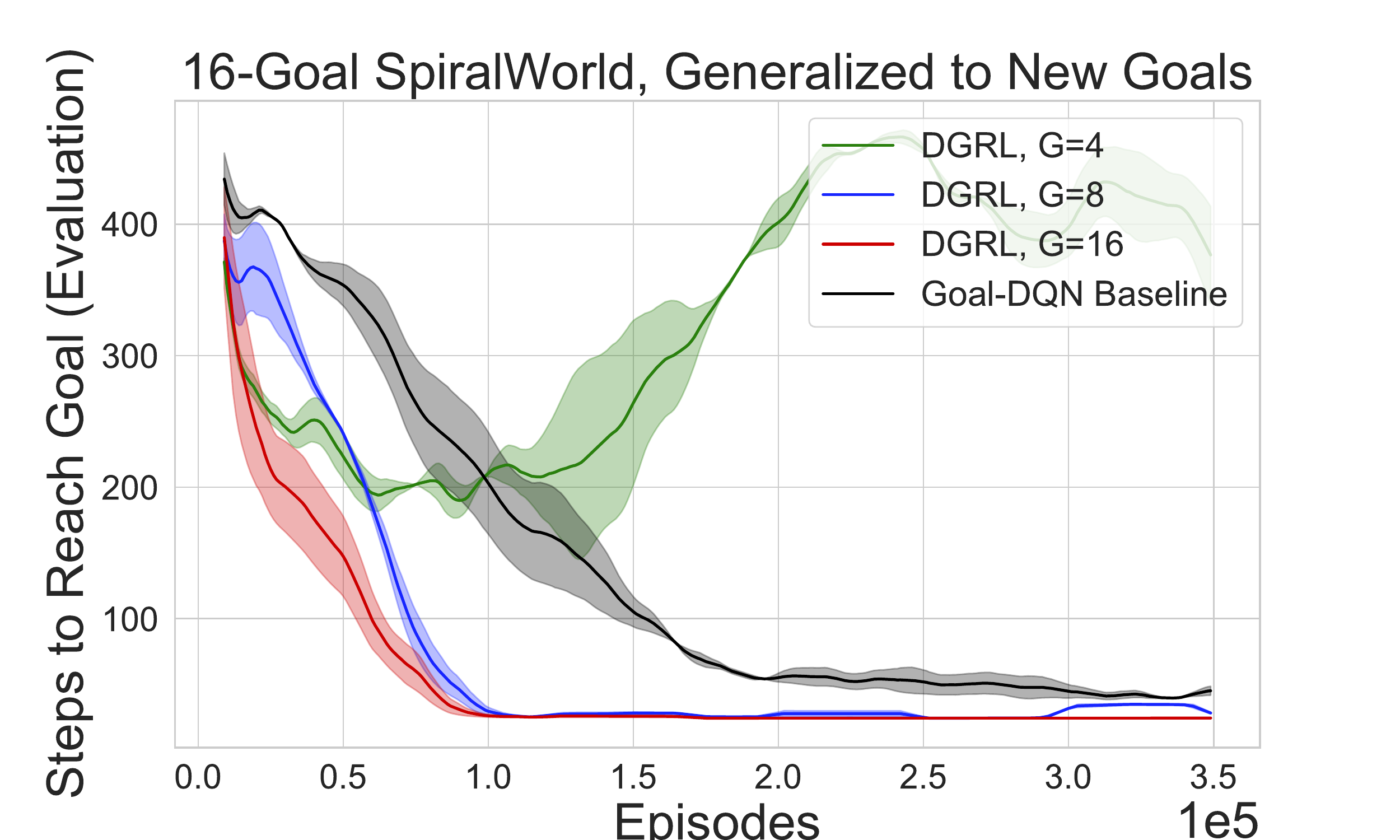} }}%
    \caption{SpiralWorld environment (left). Generalization to a test distribution containing 4-goals in a SpiralWorld environment (left). We show the total number of steps to solve all test set goals, when trained on either an 8-goal or 16-goal training distribution.}%
    \label{fig:8_ood_vis_maze}%
\end{figure}

In the previous experiment, we evaluated the generalization ability of \modelname\ by showing that learning discrete factorial representations of goals can improve generalization to novel goals. Now, we consider various setups in which a goal generating agent specifies goals using the learned codebook and a goal-conditioned agent is tasked with reaching the goals specified by the goal generating agent. We test various settings, where the goal generating agents is parameterized as an adversarial teacher \citep{AMIGO}, or as a higher-level policy in the case of hierarchical RL. 

\textbf{Procedurally Generated MiniGrid Exploration Task.}   We follow the experimental setup of \cite{AMIGO} and \cite{RIDE} and evaluate $\modelname$ on procedurally generated MiniGrid environments \cite{gym_minigrid}. In \cite{AMIGO}, a goal-generating teacher proposes goals to train a goal-conditioned ``student'' policy.  We integrate $\modelname$ on top of AMIGO \cite{AMIGO} and compare \modelname\ on a hard exploration task with state-of-the-art exploration baselines.  Experimental results are summarized in Table 1 and more details provided in Appendix~\ref{sec:app_amigo_details}. Note that unlike RIDE and RND, we do not provide an additional exploration bonus to $\modelname$, and find that $\modelname$ can still solve this hard exploration task more efficiently.

\begin{figure}
\begin{floatrow}
\ffigbox{%
    {{\includegraphics[width=0.42\textwidth]{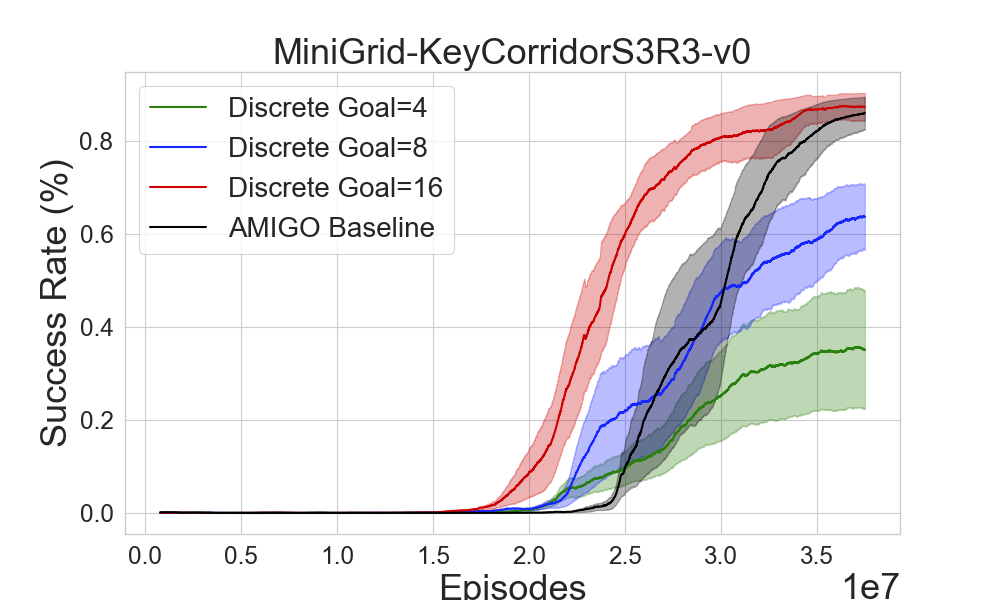} }}%
}{%
  \caption{Performance comparison of the Amigo baseline \cite[adversarially intrinsic goals]{AMIGO} with and without $\modelname$ for goal discretization.}%
}
\capbtabbox{%
    \begin{tabular}{lllllll}
    \toprule
    \textbf{Model} & \textbf{KCmedium}\\
    \midrule
    AMIGO + DGRL, G=16 & \textbf{\textcolor{red}{$.96 \pm .01$}} \\
    AMIGO + DGRL, G=8 & $.70 \pm .16$ \\
    \midrule
      \textsc{AMIGO} & \textcolor{blue}{$.93 \pm .06$}\\ 
    \textsc{RIDE} & $.90 \pm .00$\\
     \textsc{RND} & $.89 \pm .00$\\
    \textsc{ICM} & $.42 \pm .21$\\
    \bottomrule
\end{tabular}
}{%
  \caption{We added $\modelname$ on top of the Amigo baseline implementation provided by the authors. }%
}
\end{floatrow}
\end{figure}


\textbf{Goal Grounding in KeyChest Maze Navigation Domain.}
\label{sec:maze_tasks_keychest}
We consider a simple discrete state action KeyChest maze navigation task, following \cite{HRAC}, where discrete goals in the state space are provided by a higher level policy. For this task, to integrate $\modelname$, we learn an embedding $\phi(\cdot)$ of the goals, then discretize the representation with a learned codebook. We compare with a baseline HRAC \cite{HRAC} agent (details in Appendix~\ref{sec:app_keychest_details}). Figure~\ref{fig:keychest} shows an illustration of the KeyChest environment and a performance comparison of $\modelname$ with different group factors $G$. Using fewer factors ($G=4$) performs worse than the HRAC baseline, whereas using a larger number of factors ($G=8$ or $G=16$) improves the sample efficiency of the goal reaching agent, providing evidence for the benefits of factorization. 

\begin{SCfigure}[0.6]
    \centering
    {{\includegraphics[trim=0.2cm 0.2cm 0.2cm 0.2cm, clip=true, width=4cm]{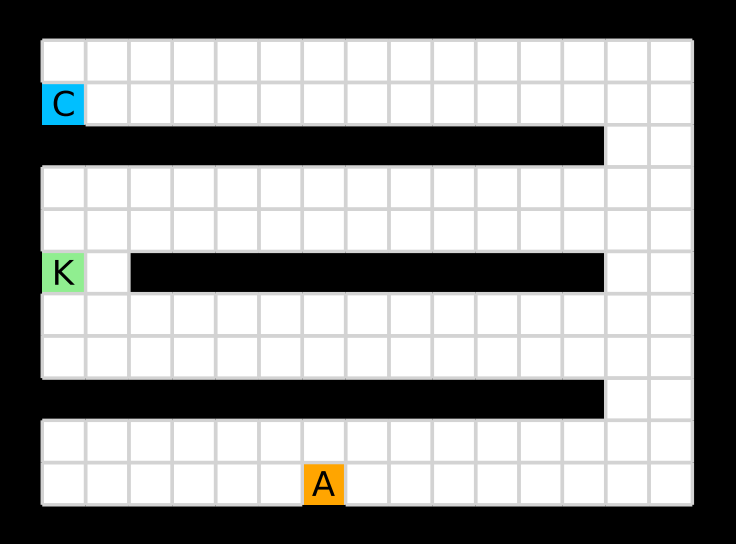} }}%
    {{\includegraphics[trim=1cm 0cm 1cm 1cm, clip=true,width=5cm]{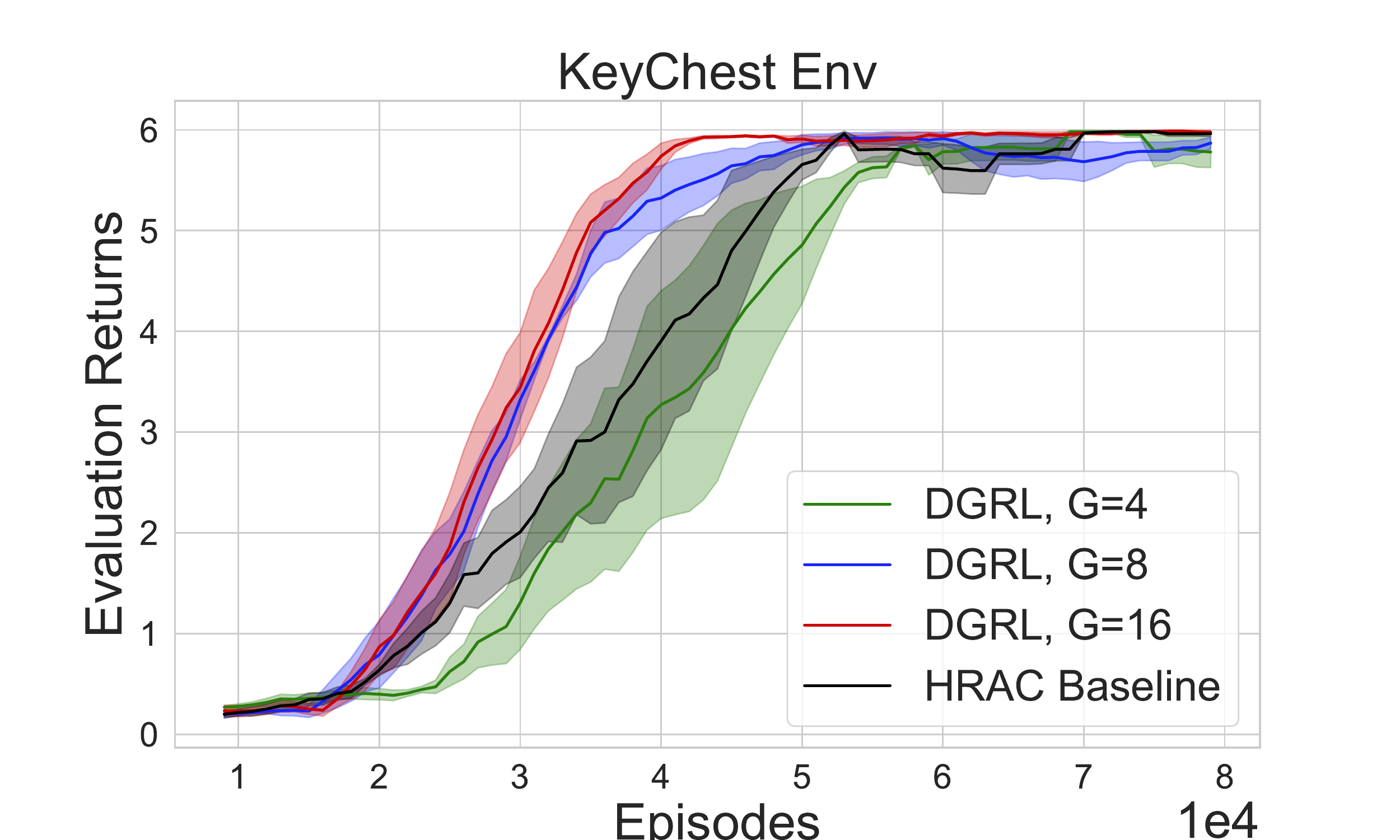} }}%
    \caption{In KeyChest, the agent (A) starts from a random stochastic position, picks up the key (K) and then uses the key to open the chest (C). We find \modelname\ improves sample efficiency over the HRAC baseline.}%
    \label{fig:keychest}%
\end{SCfigure}



\textbf{Ant Manipulation Control Domains.}
\label{sec:ant_manipulation}
We employed $\modelname$ on three different continuous control  tasks:  AntMazeSparse, AntFall and AntPush. We emphasize that these tasks are the more challenging counterparts of AntGather and AntMaze tasks, typically used in the hierarchical RL community \cite{NachumGoal, NachumGoal2}. Figure \ref{fig:8_vis_maze} provides an illustration. We evaluate goal discretization by integrating $\modelname$ to the state-of-the-art HRAC baseline. Details of the experimental setup are provided in Appendix~\ref{sec:ant_manipulation}. Figure ~\ref{fig:8_vis_maze} shows that specifying the goals using the learned codebook helps \modelname\ achieve a higher success rate compared to the HRAC baseline. 

\begin{figure}
    \centering
    {{\includegraphics[trim=1cm 0cm 1cm 1cm, clip=true, width=4.8cm]{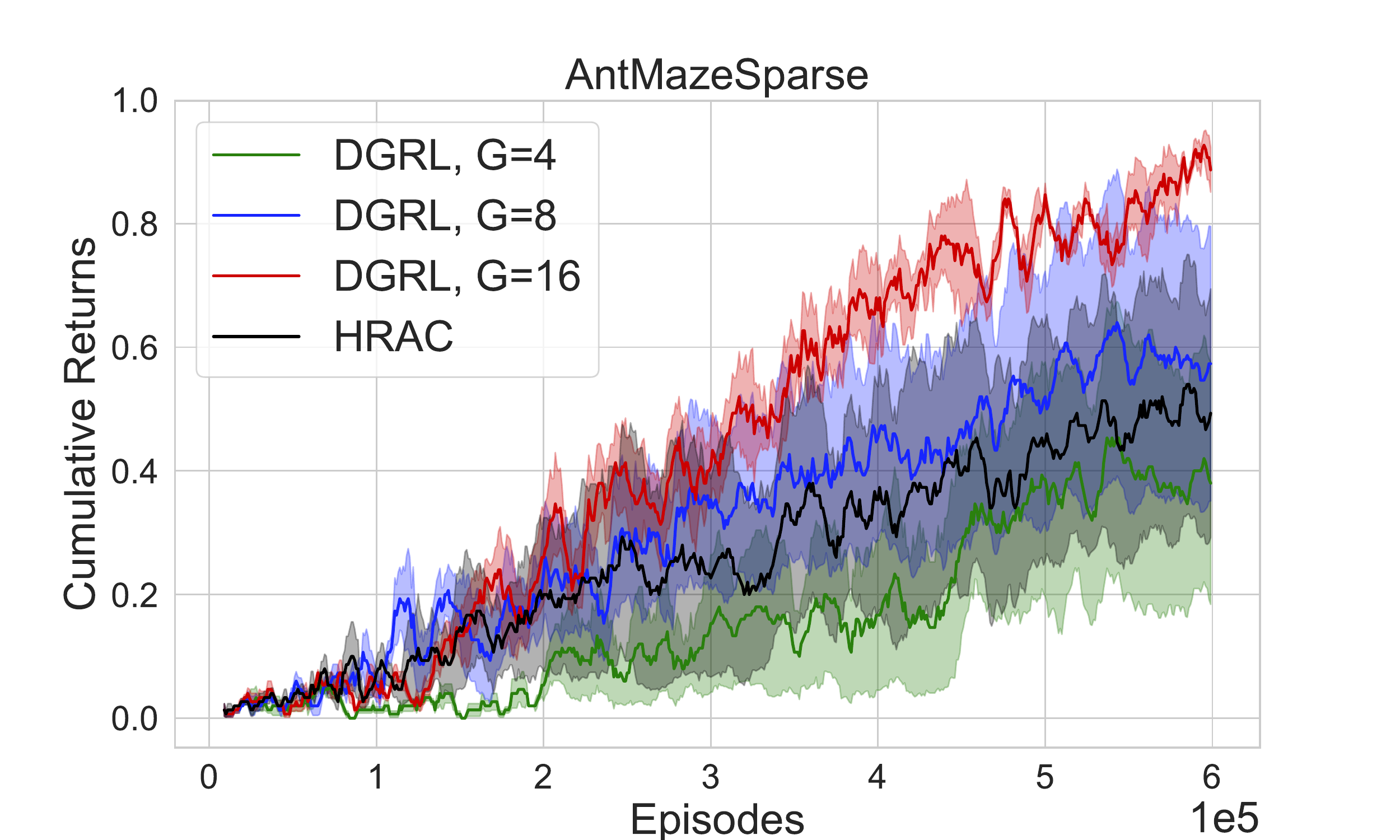} }}%
    \hspace{-0.4cm}
    {{\includegraphics[trim=1cm 0cm 1cm 1cm, clip=true,width=4.8cm]{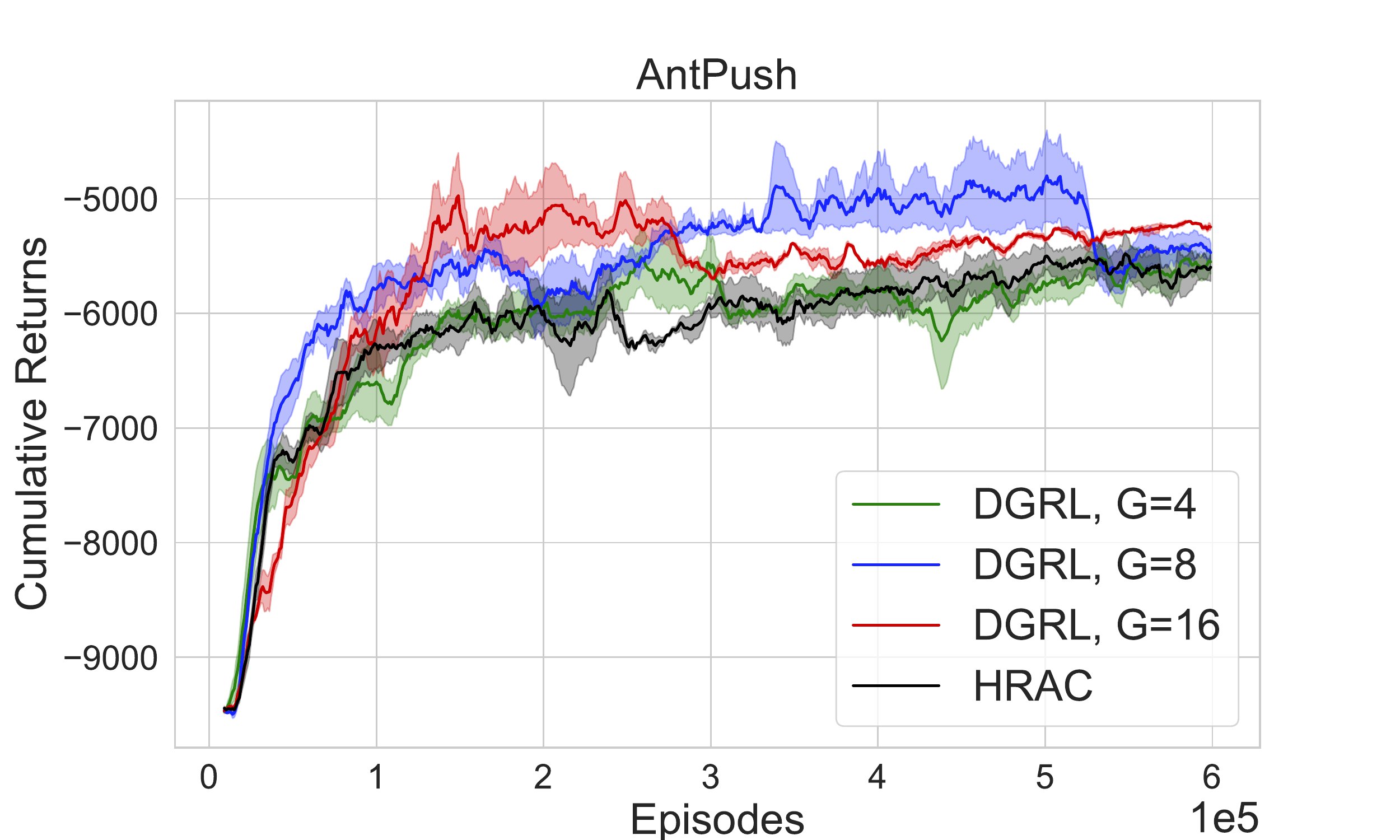} }}%
    \hspace{-0.4cm}
    {{\includegraphics[trim=1cm 0cm 1cm 1cm, clip=true,width=4.8cm]{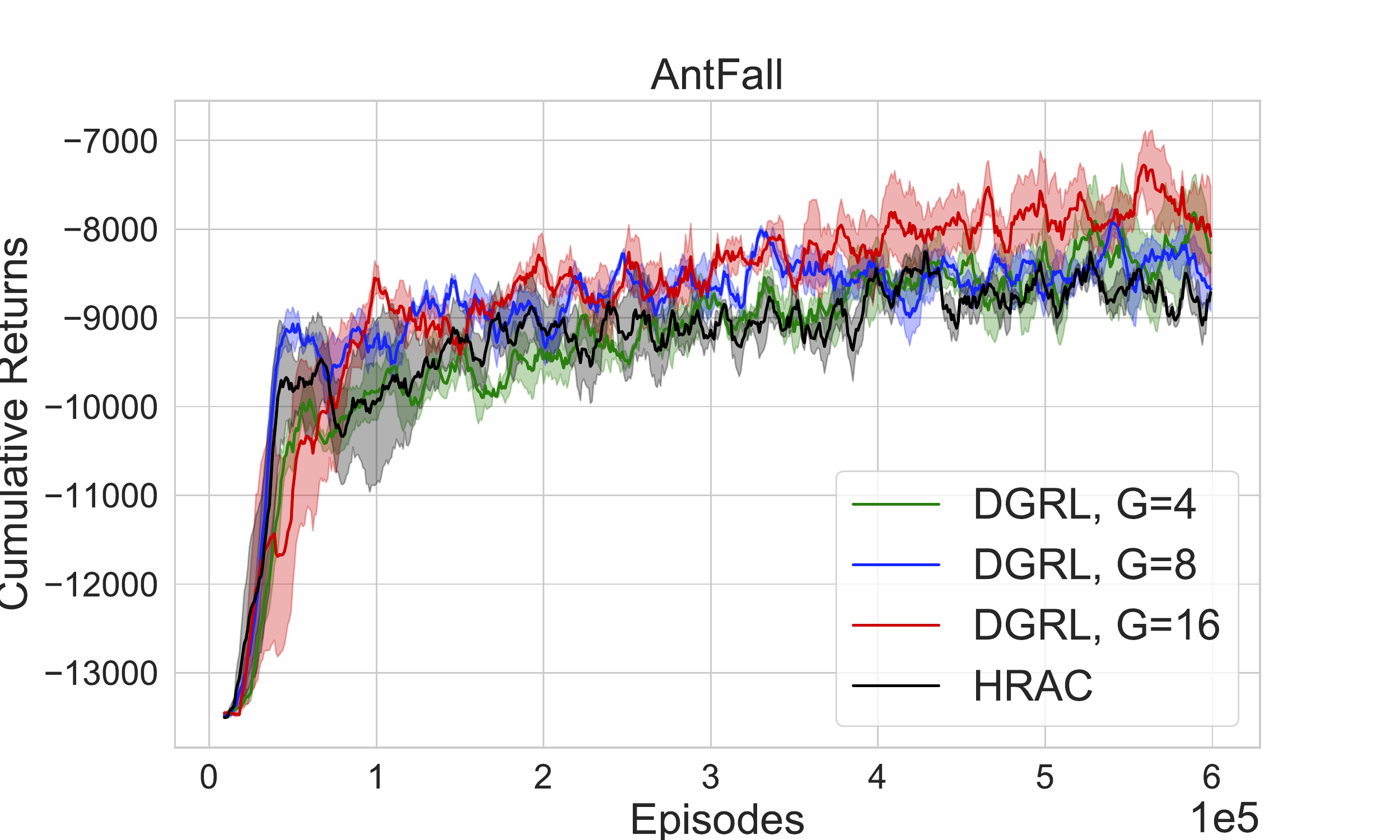} }}%
    {{\includegraphics[width=0.95\linewidth]{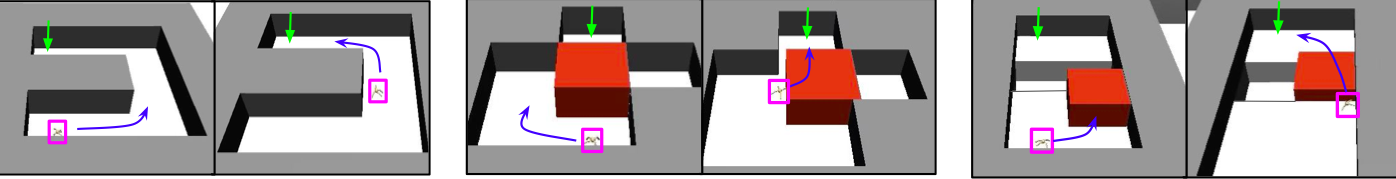} }}%
    \caption{ Comparison of $\modelname$ with baseline HRAC \cite{HRAC} on 3 different navigation tasks.}%
    \label{fig:8_vis_maze}%
\end{figure}

\textbf{Ant Navigation Maze Tasks.}
\label{sec:ant_navigation}
We consider Ant navigation tasks that require extended temporal reasoning, following the setup in Reinforcement learning with Imagined Subgoals \cite[RIS]{RIS}: a U-shaped maze, and an S-shaped maze (the S-shaped maze is shown in Figure~\ref{fig:ris_ant_maze}). The ant navigating in the maze is trained to reach any goal in the environment. The agent is evaluated for generalization in an extended temporal setting with a difficult configuration, we compare the success rate of $\modelname$ integrated on top of RIS with several baselines. We emphasize the difficulty of these tasks, where existing baselines like soft actor critic \cite[SAC]{haarnoja2018soft} and temporal difference models  \cite[TDM]{PongGDL18} fail completely. Results in Figure~\ref{fig:ris_ant_maze} show that \modelname\ improves the sample efficiency over the RIS baseline. Additional experimental setup and environment configurations are provided in Appendix~\ref{sec:app_ant_navig_details}.
\begin{figure}
    \vspace{-10mm}
    \centering
    {{\includegraphics[trim=1cm 1cm 1cm 1cm, clip=true, width=0.22\linewidth]{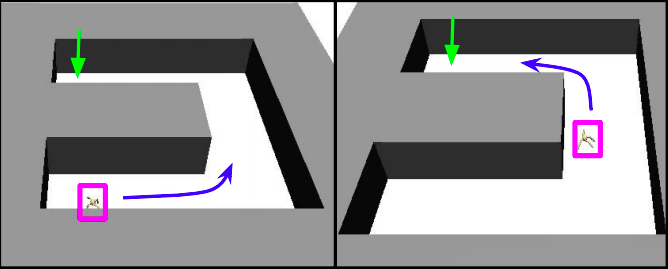} }}%
    \hspace{3mm}
    {{\includegraphics[trim=0.5cm 0cm 1.5cm 0.5cm, clip=true, width=0.35\linewidth]{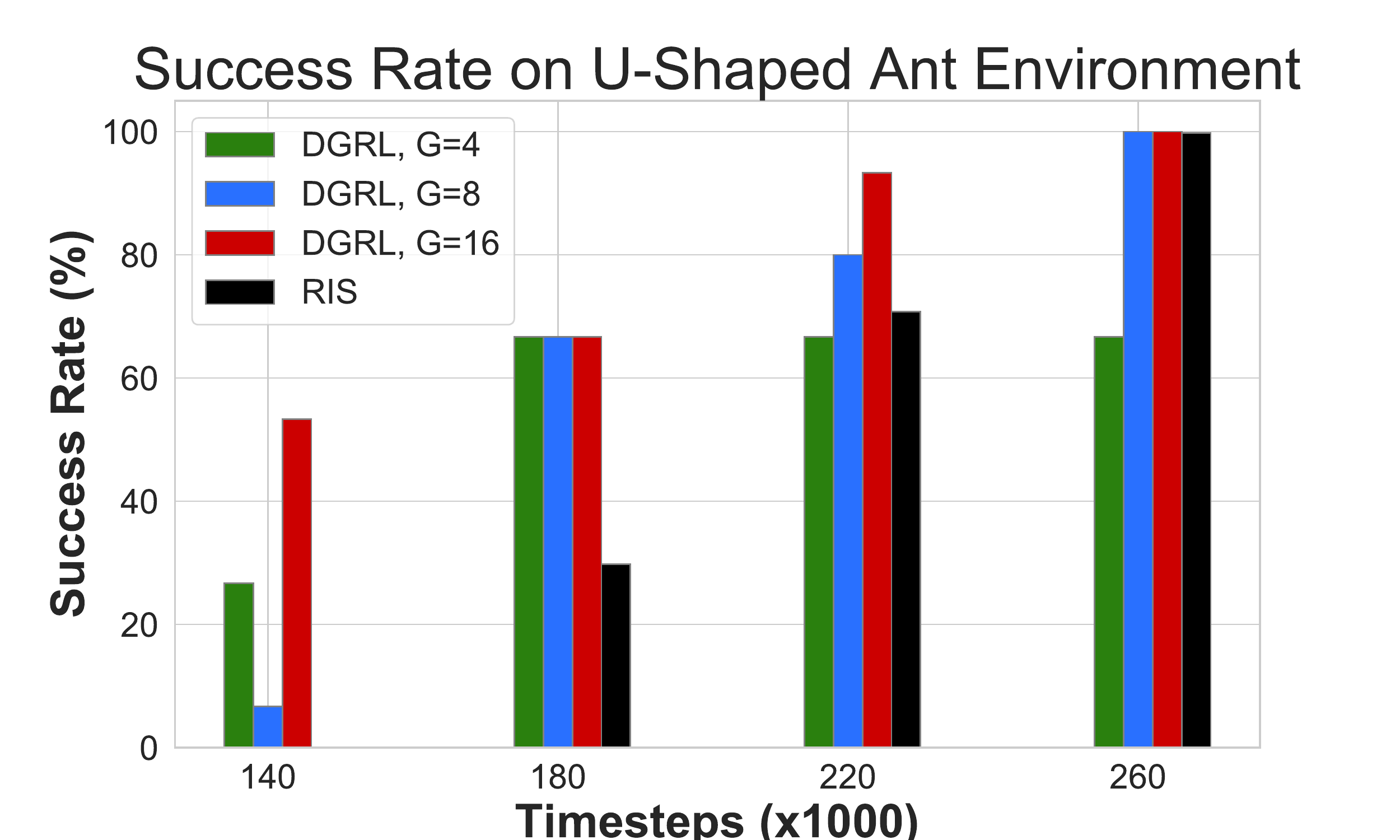} }}%
    {{\includegraphics[trim=0.0cm 0cm 1cm 0.5cm, clip=true, width=0.37\linewidth]{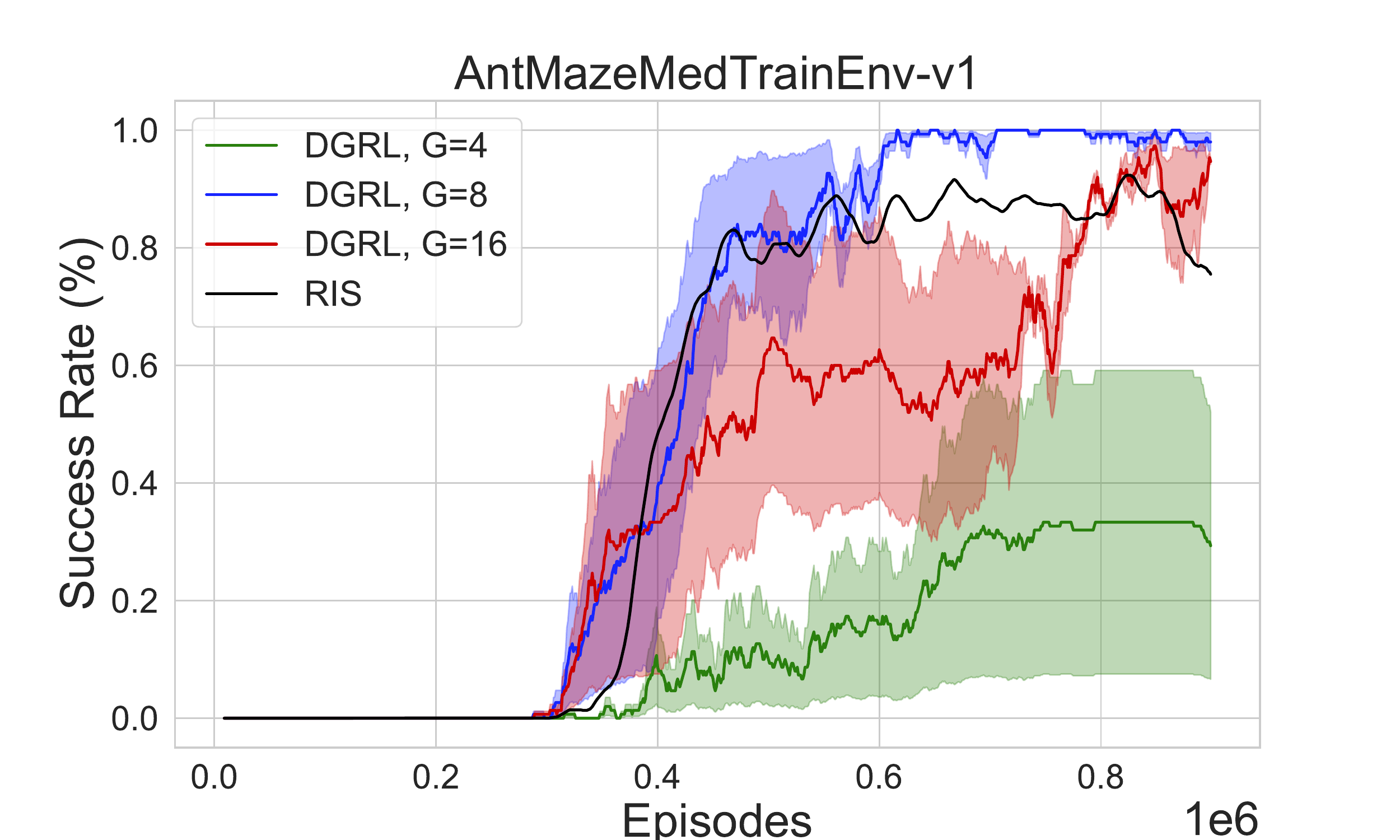} }}%
    \caption{Performance comparison with the success rate of reaching goal positions during evaluation in an extended temporal configuration of the U-shaped and maze-shaped Ant navigation tasks. We find that integrating $\modelname$ with RIS can lead to more sample efficient convergence on these tasks, while baselines such as SAC and TDM (not shown) fail completely on both AntU and AntMaze as reported by \cite{RIS}. The RIS baseline is based on raw data provided by the authors. }%
    \label{fig:ris_ant_maze}%
    \vspace*{-2mm}
\end{figure}


\textbf{Vision Based Robotic Manipulation.} Finally, we assess $\modelname$ on a hard vision-based robotic manipulation task, and use the same setup as in Section \ref{sec:ant_navigation} to integrate $\modelname$ with the state-of-the-art RIS baseline on the Sawyer task in Figure~\ref{fig:sawer_result}. This manipulation task is adapted from \cite{LEAP}, where the baseline RIS is already shown to be superior to previous goal conditioning methods. The task of the agent is to control a 2-DoF robotic arm from image input and move a puck positioned on the table. The Sawyer task is designed for training and generalization. At test time, it evaluates the agent's success at placing the puck in desired positions in a temporally extended configuration. This is a challenging vision-based complex motor task, since test time generalization requires temporally extended reasoning. Results in Figure~\ref{fig:sawer_result} show that \modelname\ improves the sample efficiency over the RIS baseline. Details of the experimental setup are provided in Appendix~\ref{sec:app_sawyer_details}.

\begin{SCfigure}[.3][h]
    \centering
    \includegraphics[trim=0.2cm 0.1cm 0.2cm 0.1cm,clip=true,width=3.8cm]{neurips2022/figures/sawyer_image.pdf}
    \includegraphics[trim=0.2cm 0.cm 0.2cm 0.1cm,clip=true, width=0.44\linewidth]{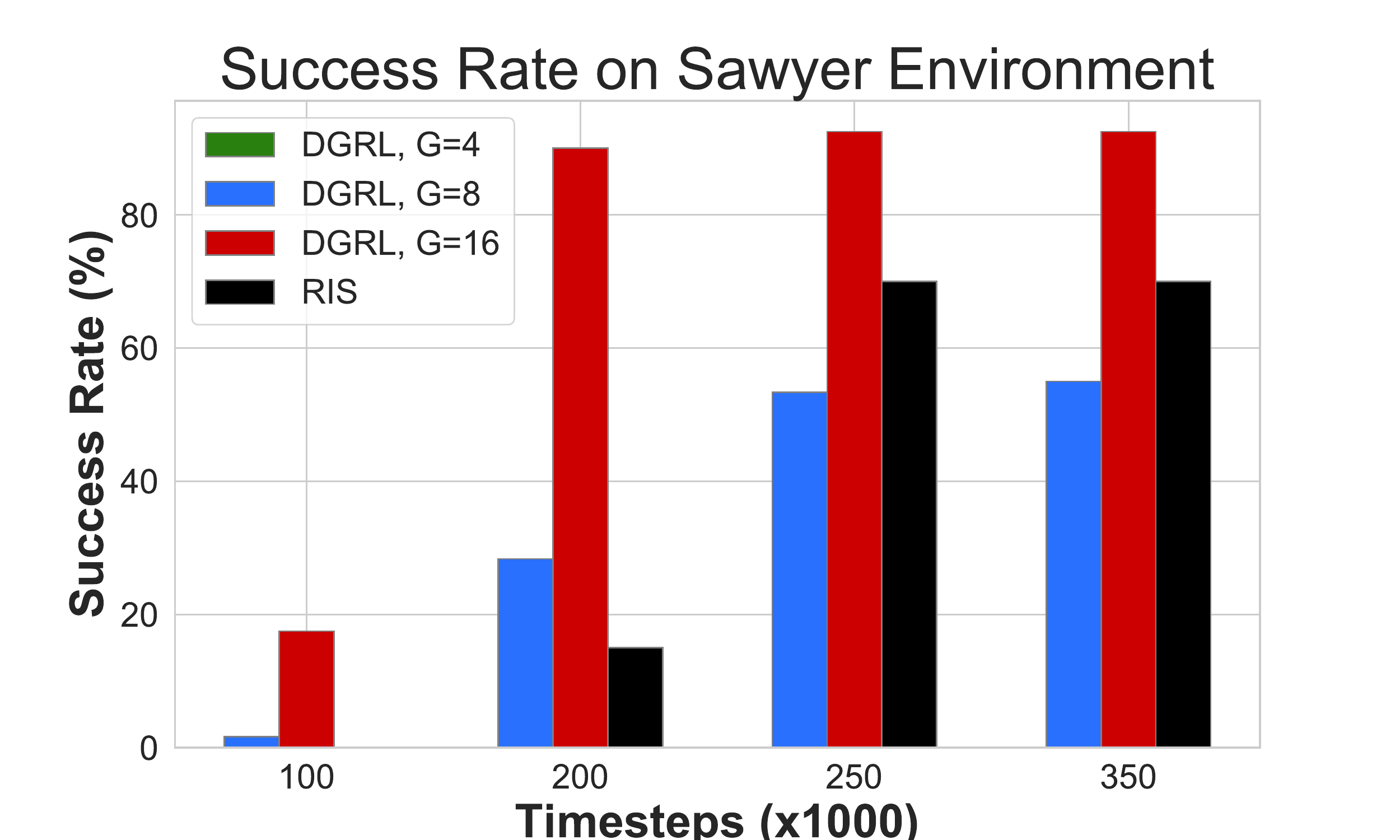} 
    \caption{Sawyer Robotic Manipulation Task. Integrating $\modelname$ with RIS (with a larger number of factors, G=8 and G=16, while G=4 fails), improves over the RIS baseline }%
    \label{fig:sawer_result}%
\end{SCfigure}

\section{Discussion}
 \vspace{-2mm}
\paragraph{Conclusion }
Our work provides direct evidence that performance of goal-conditioned RL can be improved when the representations of the goals are both \textit{discrete} and \textit{factorial}.  We show that an instantiation of this idea using multi-factor discretization significantly improves performance on a diverse set of benchmarks.

\paragraph{Limitations and Future Work}  An interesting question that arises from our work is how to theoretically \textit{ground} and \textit{specify} goals, which might be helpful for efficient structured exploration in tasks where goal seeking is crucial. Additionally, while we demonstrate that the factorial representations learnt by \modelname\ can be beneficial, it would be interesting to explore whether we can also enforce \textit{compositionality} in the latent embeddings. Finally, exploring in more details the structure of the discrete factors forming the goals, their coverage, and to which extent they semantically capture the underlying factors of the environment is a promising research avenue.

\section*{Acknowledgement}

The authors would like to thank  Nicolas Heess, John Langford, Yonathan Efroni, Manan Tomar, Dipendra Misra, Akshay Krishnamurthy, Pierre-Yves Oudeyer, Harm Van Seijen and Doina Precup for valuable discussions and insightful comments related to this work. Hongyu Zang and Xin Li were partially supported by NSFC under Grant 62276024.

\bibliography{neurips_2022}

\begin{thebibliography}{10}

\bibitem{Adhikari2020}
Ashutosh Adhikari, Xingdi Yuan, Marc-Alexandre C\^{o}t\'{e}, Mikul\'{a}\v{s}
  Zelinka, Marc-Antoine Rondeau, Romain Laroche, Pascal Poupart, Jian Tang,
  Adam Trischler, and Will Hamilton.
\newblock Learning dynamic belief graphs to generalize on text-based games.
\newblock In H.~Larochelle, M.~Ranzato, R.~Hadsell, M.F. Balcan, and H.~Lin,
  editors, {\em Advances in Neural Information Processing Systems}, volume~33,
  pages 3045--3057. Curran Associates, Inc., 2020.

\bibitem{akakzia2020grounding}
Ahmed Akakzia, C{\'e}dric Colas, Pierre-Yves Oudeyer, Mohamed Chetouani, and
  Olivier Sigaud.
\newblock Grounding language to autonomously-acquired skills via goal
  generation.
\newblock {\em arXiv preprint arXiv:2006.07185}, 2020.

\bibitem{AnandROBCH19}
Ankesh Anand, Evan Racah, Sherjil Ozair, Yoshua Bengio, Marc{-}Alexandre
  C{\^{o}}t{\'{e}}, and R.~Devon Hjelm.
\newblock Unsupervised state representation learning in atari.
\newblock In Hanna~M. Wallach, Hugo Larochelle, Alina Beygelzimer, Florence
  d'Alch{\'{e}}{-}Buc, Emily~B. Fox, and Roman Garnett, editors, {\em Advances
  in Neural Information Processing Systems 32: Annual Conference on Neural
  Information Processing Systems 2019, NeurIPS 2019, December 8-14, 2019,
  Vancouver, BC, Canada}, pages 8766--8779, 2019.

\bibitem{bahdanau2018learning}
Dzmitry Bahdanau, Felix Hill, Jan Leike, Edward Hughes, Arian Hosseini,
  Pushmeet Kohli, and Edward Grefenstette.
\newblock Learning to understand goal specifications by modelling reward.
\newblock {\em arXiv preprint arXiv:1806.01946}, 2018.

\bibitem{barto2003recent}
Andrew~G Barto and Sridhar Mahadevan.
\newblock Recent advances in hierarchical reinforcement learning.
\newblock {\em Discrete event dynamic systems}, 13(1):41--77, 2003.

\bibitem{AMIGO}
Andres Campero, Roberta Raileanu, Heinrich K{\"{u}}ttler, Joshua~B. Tenenbaum,
  Tim Rockt{\"{a}}schel, and Edward Grefenstette.
\newblock Learning with amigo: Adversarially motivated intrinsic goals.
\newblock In {\em 9th International Conference on Learning Representations,
  {ICLR} 2021, Virtual Event, Austria, May 3-7, 2021}. OpenReview.net, 2021.

\bibitem{RIS}
Elliot Chane{-}Sane, Cordelia Schmid, and Ivan Laptev.
\newblock Goal-conditioned reinforcement learning with imagined subgoals.
\newblock In Marina Meila and Tong Zhang, editors, {\em Proceedings of the 38th
  International Conference on Machine Learning, {ICML} 2021, 18-24 July 2021,
  Virtual Event}, volume 139 of {\em Proceedings of Machine Learning Research},
  pages 1430--1440. {PMLR}, 2021.

\bibitem{chao2011towards}
Crystal Chao, Maya Cakmak, and Andrea~L Thomaz.
\newblock Towards grounding concepts for transfer in goal learning from
  demonstration.
\newblock In {\em 2011 IEEE International Conference on Development and
  Learning (ICDL)}, volume~2, pages 1--6. IEEE, 2011.

\bibitem{chevalier2018babyai}
Maxime Chevalier-Boisvert, Dzmitry Bahdanau, Salem Lahlou, Lucas Willems,
  Chitwan Saharia, Thien~Huu Nguyen, and Yoshua Bengio.
\newblock Babyai: A platform to study the sample efficiency of grounded
  language learning.
\newblock {\em arXiv preprint arXiv:1810.08272}, 2018.

\bibitem{gym_minigrid}
Maxime Chevalier-Boisvert, Lucas Willems, and Suman Pal.
\newblock Minimalistic gridworld environment for openai gym.
\newblock \url{https://github.com/maximecb/gym-minigrid}, 2018.

\bibitem{dal2002planning}
Ugo Dal~Lago, Marco Pistore, and Paolo Traverso.
\newblock Planning with a language for extended goals.
\newblock In {\em AAAI/IAAI}, pages 447--454, 2002.

\bibitem{dayan1992feudal}
Peter Dayan and Geoffrey~E Hinton.
\newblock Feudal reinforcement learning.
\newblock {\em Advances in neural information processing systems}, 5, 1992.

\bibitem{DayanH92}
Peter Dayan and Geoffrey~E. Hinton.
\newblock Feudal reinforcement learning.
\newblock In Stephen~Jose Hanson, Jack~D. Cowan, and C.~Lee Giles, editors,
  {\em Advances in Neural Information Processing Systems 5, {[NIPS} Conference,
  Denver, Colorado, USA, November 30 - December 3, 1992]}, pages 271--278.
  Morgan Kaufmann, 1992.

\bibitem{Dietterich00}
Thomas~G. Dietterich.
\newblock Hierarchical reinforcement learning with the {MAXQ} value function
  decomposition.
\newblock {\em J. Artif. Intell. Res.}, 13:227--303, 2000.

\bibitem{Dwiel}
Zach Dwiel, Madhavun Candadai, Mariano~J. Phielipp, and Arjun~K. Bansal.
\newblock Hierarchical policy learning is sensitive to goal space design.
\newblock {\em CoRR}, abs/1905.01537, 2019.

\bibitem{Efroni2021ppe}
Yonathan Efroni, Dipendra Misra, Akshay Krishnamurthy, Alekh Agarwal, and John
  Langford.
\newblock Provable {RL} with exogenous distractors via multistep inverse
  dynamics.
\newblock {\em CoRR}, abs/2110.08847, 2021.

\bibitem{EysenbachSL21}
Benjamin Eysenbach, Ruslan Salakhutdinov, and Sergey Levine.
\newblock C-learning: Learning to achieve goals via recursive classification.
\newblock In {\em 9th International Conference on Learning Representations,
  {ICLR} 2021, Virtual Event, Austria, May 3-7, 2021}. OpenReview.net, 2021.

\bibitem{fu2019language}
Justin Fu, Anoop Korattikara, Sergey Levine, and Sergio Guadarrama.
\newblock From language to goals: Inverse reinforcement learning for
  vision-based instruction following.
\newblock {\em arXiv preprint arXiv:1902.07742}, 2019.

\bibitem{GoyalISALBBL19}
Anirudh Goyal, Riashat Islam, Daniel Strouse, Zafarali Ahmed, Hugo Larochelle,
  Matthew~M. Botvinick, Yoshua Bengio, and Sergey Levine.
\newblock Infobot: Transfer and exploration via the information bottleneck.
\newblock In {\em 7th International Conference on Learning Representations,
  {ICLR} 2019, New Orleans, LA, USA, May 6-9, 2019}. OpenReview.net, 2019.

\bibitem{grice1975logic}
Herbert~P Grice.
\newblock Logic and conversation.
\newblock In {\em Speech acts}, pages 41--58. Brill, 1975.

\bibitem{haarnoja2018soft}
Tuomas Haarnoja, Aurick Zhou, Pieter Abbeel, and Sergey Levine.
\newblock Soft actor-critic: Off-policy maximum entropy deep reinforcement
  learning with a stochastic actor.
\newblock In {\em International conference on machine learning}, pages
  1861--1870. PMLR, 2018.

\bibitem{jennrich1969asymptotic}
Robert~I Jennrich.
\newblock Asymptotic properties of non-linear least squares estimators.
\newblock {\em The Annals of Mathematical Statistics}, 40(2):633--643, 1969.

\bibitem{jiang2019language}
Yiding Jiang, Shixiang~Shane Gu, Kevin~P Murphy, and Chelsea Finn.
\newblock Language as an abstraction for hierarchical deep reinforcement
  learning.
\newblock {\em Advances in Neural Information Processing Systems}, 32, 2019.

\bibitem{Kaelbling93}
Leslie~Pack Kaelbling.
\newblock Learning to achieve goals.
\newblock In Ruzena Bajcsy, editor, {\em Proceedings of the 13th International
  Joint Conference on Artificial Intelligence. Chamb{\'{e}}ry, France, August
  28 - September 3, 1993}, pages 1094--1099. Morgan Kaufmann, 1993.

\bibitem{kaelbling1987architecture}
Leslie~Pack Kaelbling et~al.
\newblock An architecture for intelligent reactive systems.
\newblock {\em Reasoning about actions and plans}, pages 395--410, 1987.

\bibitem{kaelbling1996reinforcement}
Leslie~Pack Kaelbling, Michael~L Littman, and Andrew~W Moore.
\newblock Reinforcement learning: A survey.
\newblock {\em Journal of artificial intelligence research}, 4:237--285, 1996.

\bibitem{khetarpal2020can}
Khimya Khetarpal, Zafarali Ahmed, Gheorghe Comanici, David Abel, and Doina
  Precup.
\newblock What can i do here? a theory of affordances in reinforcement
  learning.
\newblock In {\em International Conference on Machine Learning}, pages
  5243--5253. PMLR, 2020.

\bibitem{kulkarni2016hierarchical}
Tejas~D Kulkarni, Karthik Narasimhan, Ardavan Saeedi, and Josh Tenenbaum.
\newblock Hierarchical deep reinforcement learning: Integrating temporal
  abstraction and intrinsic motivation.
\newblock {\em Advances in neural information processing systems}, 29, 2016.

\bibitem{lamb2022guaranteed}
Alex Lamb, Riashat Islam, Yonathan Efroni, Aniket Didolkar, Dipendra Misra,
  Dylan Foster, Lekan Molu, Rajan Chari, Akshay Krishnamurthy, and John
  Langford.
\newblock Guaranteed discovery of controllable latent states with multi-step
  inverse models.
\newblock {\em arXiv preprint arXiv:2207.08229}, 2022.

\bibitem{Laroche2017}
Romain Laroche and Merwan Barlier.
\newblock Transfer reinforcement learning with shared dynamics.
\newblock In {\em Thirty-First AAAI Conference on Artificial Intelligence},
  2017.

\bibitem{Laroche2021}
Romain Laroche and Remi Tachet~des Combes.
\newblock Dr jekyll \& mr hyde: the strange case of off-policy policy updates.
\newblock {\em Advances in Neural Information Processing Systems},
  34:24442--24454, 2021.

\bibitem{laskin2020curl}
Michael Laskin, Aravind Srinivas, and Pieter Abbeel.
\newblock Curl: Contrastive unsupervised representations for reinforcement
  learning.
\newblock In {\em International Conference on Machine Learning}, pages
  5639--5650. PMLR, 2020.

\bibitem{LevyKPS19}
Andrew Levy, George~Dimitri Konidaris, Robert~Platt Jr., and Kate Saenko.
\newblock Learning multi-level hierarchies with hindsight.
\newblock In {\em 7th International Conference on Learning Representations,
  {ICLR} 2019, New Orleans, LA, USA, May 6-9, 2019}. OpenReview.net, 2019.

\bibitem{liu2021discrete}
Dianbo Liu, Alex~M Lamb, Kenji Kawaguchi, Anirudh Goyal, Chen Sun, Michael~C
  Mozer, and Yoshua Bengio.
\newblock Discrete-valued neural communication.
\newblock {\em Advances in Neural Information Processing Systems}, 34, 2021.

\bibitem{mazoure2020deep}
Bogdan Mazoure, Remi Tachet~des Combes, Thang~Long Doan, Philip Bachman, and
  R~Devon Hjelm.
\newblock Deep reinforcement and infomax learning.
\newblock {\em Advances in Neural Information Processing Systems},
  33:3686--3698, 2020.

\bibitem{MisraHK020}
Dipendra Misra, Mikael Henaff, Akshay Krishnamurthy, and John Langford.
\newblock Kinematic state abstraction and provably efficient rich-observation
  reinforcement learning.
\newblock In {\em Proceedings of the 37th International Conference on Machine
  Learning, {ICML} 2020, 13-18 July 2020, Virtual Event}, volume 119 of {\em
  Proceedings of Machine Learning Research}, pages 6961--6971. {PMLR}, 2020.

\bibitem{mnih2013playing}
Volodymyr Mnih, Koray Kavukcuoglu, David Silver, Alex Graves, Ioannis
  Antonoglou, Daan Wierstra, and Martin Riedmiller.
\newblock Playing atari with deep reinforcement learning.
\newblock {\em arXiv preprint arXiv:1312.5602}, 2013.

\bibitem{moore1993prioritized}
Andrew~W Moore and Christopher~G Atkeson.
\newblock Prioritized sweeping: Reinforcement learning with less data and less
  time.
\newblock {\em Machine learning}, 13(1):103--130, 1993.

\bibitem{NachumGoal}
Ofir Nachum, Shixiang Gu, Honglak Lee, and Sergey Levine.
\newblock Data-efficient hierarchical reinforcement learning.
\newblock In Samy Bengio, Hanna~M. Wallach, Hugo Larochelle, Kristen Grauman,
  Nicol{\`{o}} Cesa{-}Bianchi, and Roman Garnett, editors, {\em Advances in
  Neural Information Processing Systems 31: Annual Conference on Neural
  Information Processing Systems 2018, NeurIPS 2018, December 3-8, 2018,
  Montr{\'{e}}al, Canada}, pages 3307--3317, 2018.

\bibitem{nachum2018near}
Ofir Nachum, Shixiang Gu, Honglak Lee, and Sergey Levine.
\newblock Near-optimal representation learning for hierarchical reinforcement
  learning.
\newblock {\em arXiv preprint arXiv:1810.01257}, 2018.

\bibitem{NachumGoal2}
Ofir Nachum, Haoran Tang, Xingyu Lu, Shixiang Gu, Honglak Lee, and Sergey
  Levine.
\newblock Why does hierarchy (sometimes) work so well in reinforcement
  learning?
\newblock {\em CoRR}, abs/1909.10618, 2019.

\bibitem{NairF20}
Suraj Nair and Chelsea Finn.
\newblock Hierarchical foresight: Self-supervised learning of long-horizon
  tasks via visual subgoal generation.
\newblock In {\em 8th International Conference on Learning Representations,
  {ICLR} 2020, Addis Ababa, Ethiopia, April 26-30, 2020}. OpenReview.net, 2020.

\bibitem{LEAP}
Soroush Nasiriany, Vitchyr Pong, Steven Lin, and Sergey Levine.
\newblock Planning with goal-conditioned policies.
\newblock In Hanna~M. Wallach, Hugo Larochelle, Alina Beygelzimer, Florence
  d'Alch{\'{e}}{-}Buc, Emily~B. Fox, and Roman Garnett, editors, {\em Advances
  in Neural Information Processing Systems 32: Annual Conference on Neural
  Information Processing Systems 2019, NeurIPS 2019, December 8-14, 2019,
  Vancouver, BC, Canada}, pages 14814--14825, 2019.

\bibitem{plappert2018multi}
Matthias Plappert, Marcin Andrychowicz, Alex Ray, Bob McGrew, Bowen Baker,
  Glenn Powell, Jonas Schneider, Josh Tobin, Maciek Chociej, Peter Welinder,
  et~al.
\newblock Multi-goal reinforcement learning: Challenging robotics environments
  and request for research.
\newblock {\em arXiv preprint arXiv:1802.09464}, 2018.

\bibitem{pong2018temporal}
Vitchyr Pong, Shixiang Gu, Murtaza Dalal, and Sergey Levine.
\newblock Temporal difference models: Model-free deep rl for model-based
  control.
\newblock {\em arXiv preprint arXiv:1802.09081}, 2018.

\bibitem{PongGDL18}
Vitchyr Pong, Shixiang Gu, Murtaza Dalal, and Sergey Levine.
\newblock Temporal difference models: Model-free deep {RL} for model-based
  control.
\newblock In {\em 6th International Conference on Learning Representations,
  {ICLR} 2018, Vancouver, BC, Canada, April 30 - May 3, 2018, Conference Track
  Proceedings}. OpenReview.net, 2018.

\bibitem{RIDE}
Roberta Raileanu and Tim Rockt{\"{a}}schel.
\newblock {RIDE:} rewarding impact-driven exploration for
  procedurally-generated environments.
\newblock In {\em 8th International Conference on Learning Representations,
  {ICLR} 2020, Addis Ababa, Ethiopia, April 26-30, 2020}. OpenReview.net, 2020.

\bibitem{razavi2019generating}
Ali Razavi, Aaron Van~den Oord, and Oriol Vinyals.
\newblock Generating diverse high-fidelity images with vq-vae-2.
\newblock {\em Advances in neural information processing systems}, 32, 2019.

\bibitem{schaul2015universal}
Tom Schaul, Daniel Horgan, Karol Gregor, and David Silver.
\newblock Universal value function approximators.
\newblock In {\em International conference on machine learning}, pages
  1312--1320. PMLR, 2015.

\bibitem{SchaulHGS15}
Tom Schaul, Daniel Horgan, Karol Gregor, and David Silver.
\newblock Universal value function approximators.
\newblock In Francis~R. Bach and David~M. Blei, editors, {\em Proceedings of
  the 32nd International Conference on Machine Learning, {ICML} 2015, Lille,
  France, 6-11 July 2015}, volume~37 of {\em {JMLR} Workshop and Conference
  Proceedings}, pages 1312--1320. JMLR.org, 2015.

\bibitem{schulman2015trust}
John Schulman, Sergey Levine, Pieter Abbeel, Michael Jordan, and Philipp
  Moritz.
\newblock Trust region policy optimization.
\newblock In {\em International conference on machine learning}, pages
  1889--1897. PMLR, 2015.

\bibitem{schulman2017proximal}
John Schulman, Filip Wolski, Prafulla Dhariwal, Alec Radford, and Oleg Klimov.
\newblock Proximal policy optimization algorithms.
\newblock {\em arXiv preprint arXiv:1707.06347}, 2017.

\bibitem{schultz1998predictive}
Wolfram Schultz.
\newblock Predictive reward signal of dopamine neurons.
\newblock {\em Journal of neurophysiology}, 80(1):1--27, 1998.

\bibitem{schwarzer2020data}
Max Schwarzer, Ankesh Anand, Rishab Goel, R~Devon Hjelm, Aaron Courville, and
  Philip Bachman.
\newblock Data-efficient reinforcement learning with self-predictive
  representations.
\newblock {\em arXiv preprint arXiv:2007.05929}, 2020.

\bibitem{SchwarzerAGHCB21}
Max Schwarzer, Ankesh Anand, Rishab Goel, R.~Devon Hjelm, Aaron~C. Courville,
  and Philip Bachman.
\newblock Data-efficient reinforcement learning with self-predictive
  representations.
\newblock In {\em 9th International Conference on Learning Representations,
  {ICLR} 2021, Virtual Event, Austria, May 3-7, 2021}. OpenReview.net, 2021.

\bibitem{SchwarzerRNACHB21}
Max Schwarzer, Nitarshan Rajkumar, Michael Noukhovitch, Ankesh Anand, Laurent
  Charlin, R.~Devon Hjelm, Philip Bachman, and Aaron~C. Courville.
\newblock Pretraining representations for data-efficient reinforcement
  learning.
\newblock In Marc'Aurelio Ranzato, Alina Beygelzimer, Yann~N. Dauphin, Percy
  Liang, and Jennifer~Wortman Vaughan, editors, {\em Advances in Neural
  Information Processing Systems 34: Annual Conference on Neural Information
  Processing Systems 2021, NeurIPS 2021, December 6-14, 2021, virtual}, pages
  12686--12699, 2021.

\bibitem{silver2021reward}
David Silver, Satinder Singh, Doina Precup, and Richard~S Sutton.
\newblock Reward is enough.
\newblock {\em Artificial Intelligence}, 299:103535, 2021.

\bibitem{singh1996reinforcement}
Satinder~P Singh and Richard~S Sutton.
\newblock Reinforcement learning with replacing eligibility traces.
\newblock {\em Machine learning}, 22(1):123--158, 1996.

\bibitem{stooke2021decoupling}
Adam Stooke, Kimin Lee, Pieter Abbeel, and Michael Laskin.
\newblock Decoupling representation learning from reinforcement learning.
\newblock In {\em International Conference on Machine Learning}, pages
  9870--9879. PMLR, 2021.

\bibitem{Sukhbaatar}
Sainbayar Sukhbaatar, Emily Denton, Arthur Szlam, and Rob Fergus.
\newblock Learning goal embeddings via self-play for hierarchical reinforcement
  learning.
\newblock {\em CoRR}, abs/1811.09083, 2018.

\bibitem{sutton2018reinforcement}
Richard~S Sutton and Andrew~G Barto.
\newblock {\em Reinforcement learning: An introduction}.
\newblock MIT press, 2018.

\bibitem{sutton2011horde}
Richard~S Sutton, Joseph Modayil, Michael Delp, Thomas Degris, Patrick~M
  Pilarski, Adam White, and Doina Precup.
\newblock Horde: A scalable real-time architecture for learning knowledge from
  unsupervised sensorimotor interaction.
\newblock In {\em The 10th International Conference on Autonomous Agents and
  Multiagent Systems-Volume 2}, pages 761--768, 2011.

\bibitem{sutton1984temporal}
Richard~Stuart Sutton.
\newblock {\em Temporal credit assignment in reinforcement learning}.
\newblock PhD thesis, University of Massachusetts Amherst, 1984.

\bibitem{tesauro1991practical}
Gerald Tesauro.
\newblock Practical issues in temporal difference learning.
\newblock {\em Advances in neural information processing systems}, 4, 1991.

\bibitem{van2017neural}
Aaron Van Den~Oord, Oriol Vinyals, et~al.
\newblock Neural discrete representation learning.
\newblock {\em Advances in neural information processing systems}, 30, 2017.

\bibitem{van1996}
Aad~W. {van der Vaart} and Jon~A. Wellner.
\newblock {\em Weak Convergence and Empirical Processes}.
\newblock Springer New York, 1996.

\bibitem{Veeriah_ManyGoals}
Vivek Veeriah, Junhyuk Oh, and Satinder Singh.
\newblock Many-goals reinforcement learning.
\newblock {\em CoRR}, abs/1806.09605, 2018.

\bibitem{Weir2022}
Nathaniel Weir, Xingdi Yuan, Marc-Alexandre Côté, Matthew Hausknecht, Romain
  Laroche, Ida Momennejad, Harm Van~Seijen, and Benjamin Van~Durme.
\newblock One-shot learning from a demonstration with hierarchical latent
  language, 2022.

\bibitem{WieringS97}
Marco~A. Wiering and J{\"{u}}rgen Schmidhuber.
\newblock Hq-learning.
\newblock {\em Adapt. Behav.}, 6(2):219--246, 1997.

\bibitem{YangN21}
Mengjiao Yang and Ofir Nachum.
\newblock Representation matters: Offline pretraining for sequential decision
  making.
\newblock In Marina Meila and Tong Zhang, editors, {\em Proceedings of the 38th
  International Conference on Machine Learning, {ICML} 2021, 18-24 July 2021,
  Virtual Event}, volume 139 of {\em Proceedings of Machine Learning Research},
  pages 11784--11794. {PMLR}, 2021.

\bibitem{YaratsFLP21}
Denis Yarats, Rob Fergus, Alessandro Lazaric, and Lerrel Pinto.
\newblock Reinforcement learning with prototypical representations.
\newblock In Marina Meila and Tong Zhang, editors, {\em Proceedings of the 38th
  International Conference on Machine Learning, {ICML} 2021, 18-24 July 2021,
  Virtual Event}, volume 139 of {\em Proceedings of Machine Learning Research},
  pages 11920--11931. {PMLR}, 2021.

\bibitem{Zhang0S21a}
Lunjun Zhang, Ge~Yang, and Bradly~C. Stadie.
\newblock World model as a graph: Learning latent landmarks for planning.
\newblock In Marina Meila and Tong Zhang, editors, {\em Proceedings of the 38th
  International Conference on Machine Learning, {ICML} 2021, 18-24 July 2021,
  Virtual Event}, volume 139 of {\em Proceedings of Machine Learning Research},
  pages 12611--12620. {PMLR}, 2021.

\bibitem{HRAC}
Tianren Zhang, Shangqi Guo, Tian Tan, Xiaolin Hu, and Feng Chen.
\newblock Generating adjacency-constrained subgoals in hierarchical
  reinforcement learning.
\newblock In Hugo Larochelle, Marc'Aurelio Ranzato, Raia Hadsell,
  Maria{-}Florina Balcan, and Hsuan{-}Tien Lin, editors, {\em Advances in
  Neural Information Processing Systems 33: Annual Conference on Neural
  Information Processing Systems 2020, NeurIPS 2020, December 6-12, 2020,
  virtual}, 2020.

\end{thebibliography}

\bibliographystyle{plain}

\newpage
\appendix

\section{Appendix}
In section \ref{ref:related_work}, we discuss extended related work. In section \ref{sec:app_add_exp_results} we include details of our experiment setup and additional ablation analysis evaluating $\modelname$. Section \ref{sec:app_theory} contains detailed proof for theorem \ref{thm:1}. Finally, we provide our algorithm, $\modelname$ that requires minimal changes to existing goal conditioned RL setup, in section \ref{sec:app_algo}, with sample code snippet provided in \ref{sec:app_code}

\section{Demonstrating Factorized Representations in a Robot Experiment with Visual Background Distractors}

We  demonstrate the ability of $\modelname$ to learn factorial representations on real world robot data, where the robot arm moves in presence of background video distractors \cite{lamb2022guaranteed}. Further details on the robot arm data collection are provided below. The data contains rich temporal background noise. We first learn a representation with a simple auto-encoder, following by the discretization bottleneck of $\modelname$, and then reconstruct the image with different discrete factors. Figure \ref{fig:robot_data_compositionality} demonstrates factorization in the learnt representation tweaking the different factors used in the discretization bottleneck. In particular, we observe some form of ``compositionality'' emerging as the decoder was never trained on some of the combinations of factors, for instance (person in the background + orange lamp) and (person in the background + arm to the left).

\begin{figure}[!htbp]
\centering
\subfigure[Normal Reconstructions]{
\includegraphics[trim=0cm 27.5cm 36.4cm 0cm, clip=true,width=0.9\textwidth]{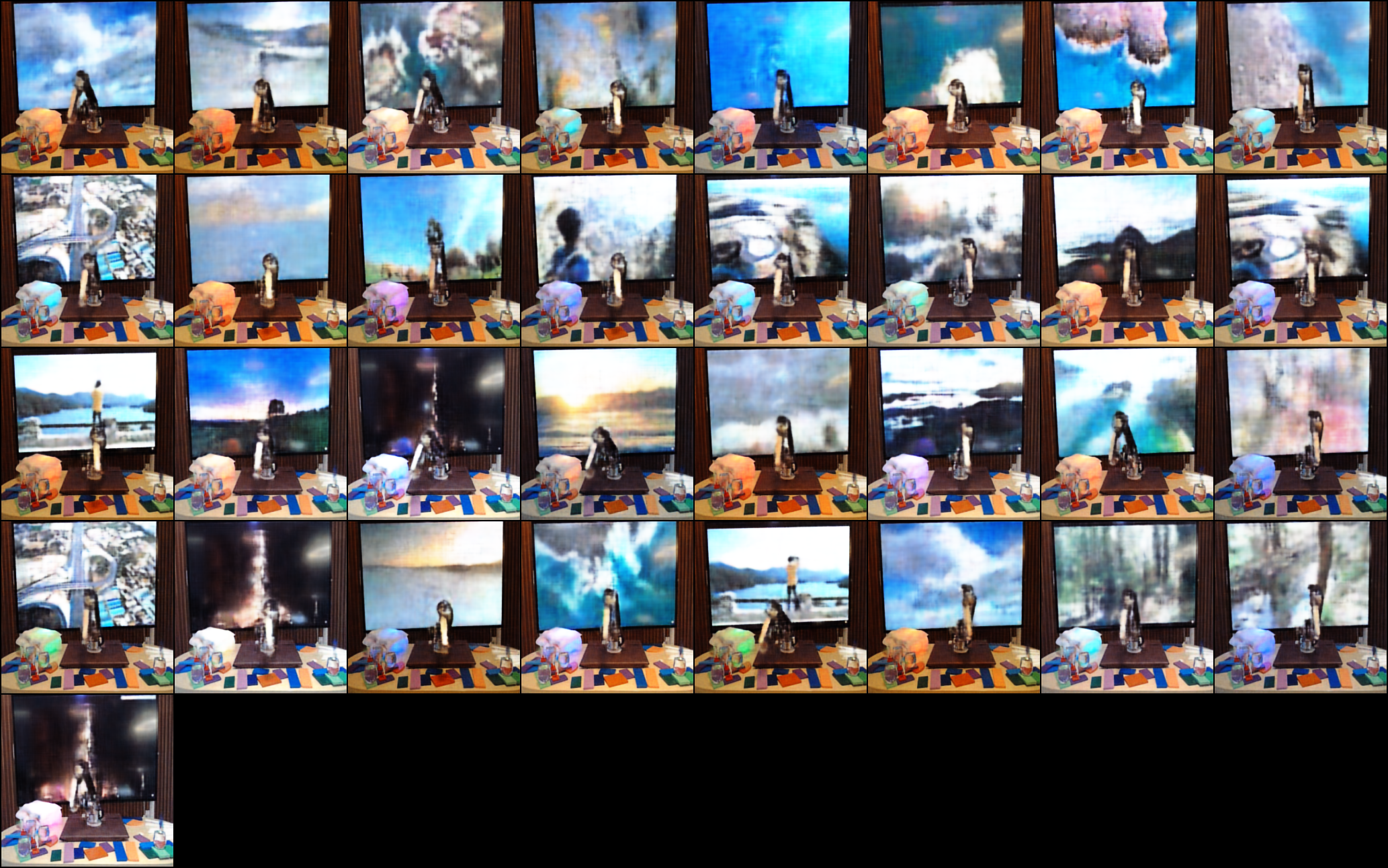}}
\subfigure[Change Lamp to Orange (changing a single factor)]{
\includegraphics[trim=0cm 27.5cm 36.4cm 0cm, clip=true,width=0.9\textwidth]{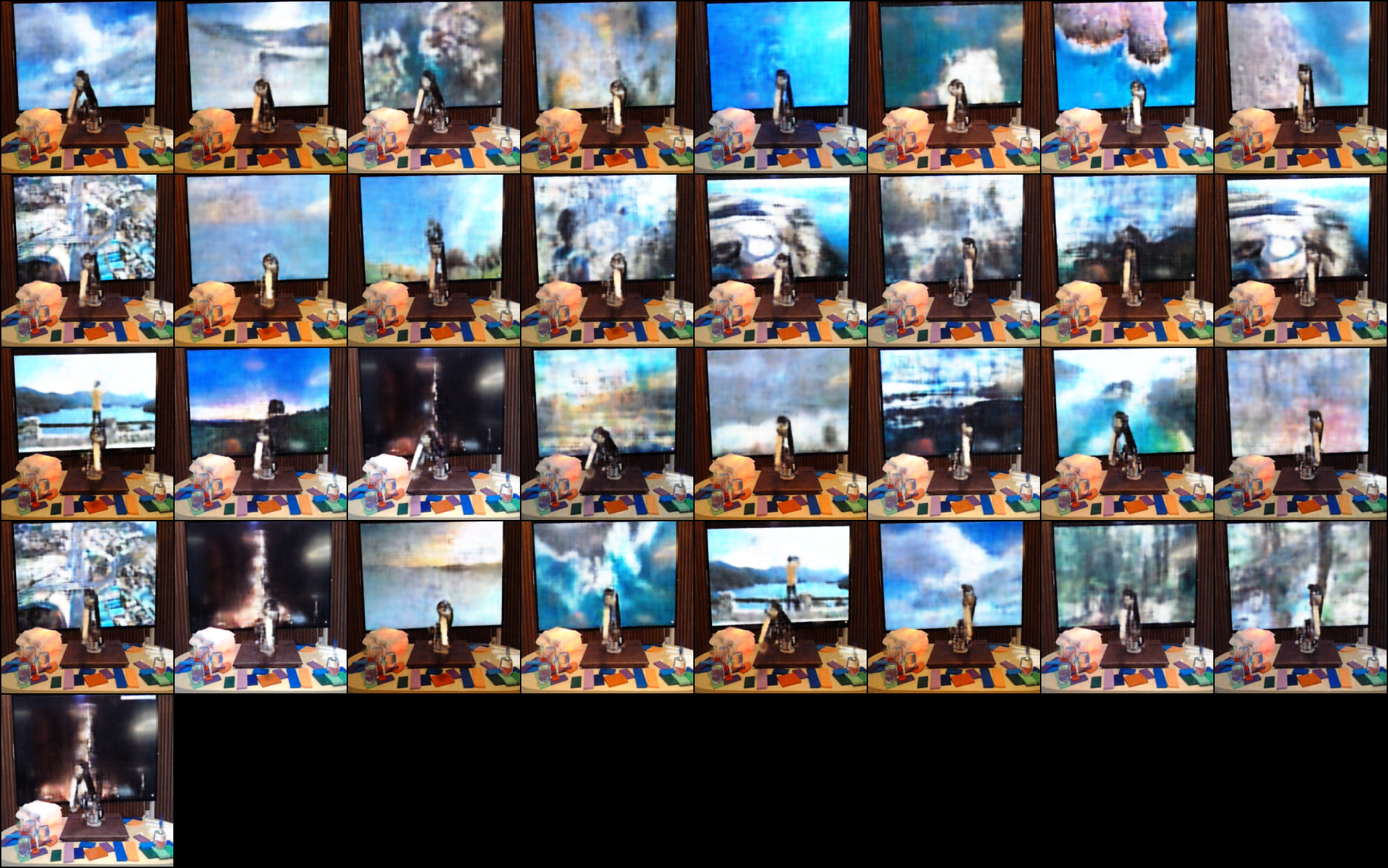}}
\subfigure[Make arm point left (changing a single factor)]{
\includegraphics[trim=0cm 27.5cm 36.4cm 0cm, clip=true,width=0.9\textwidth]{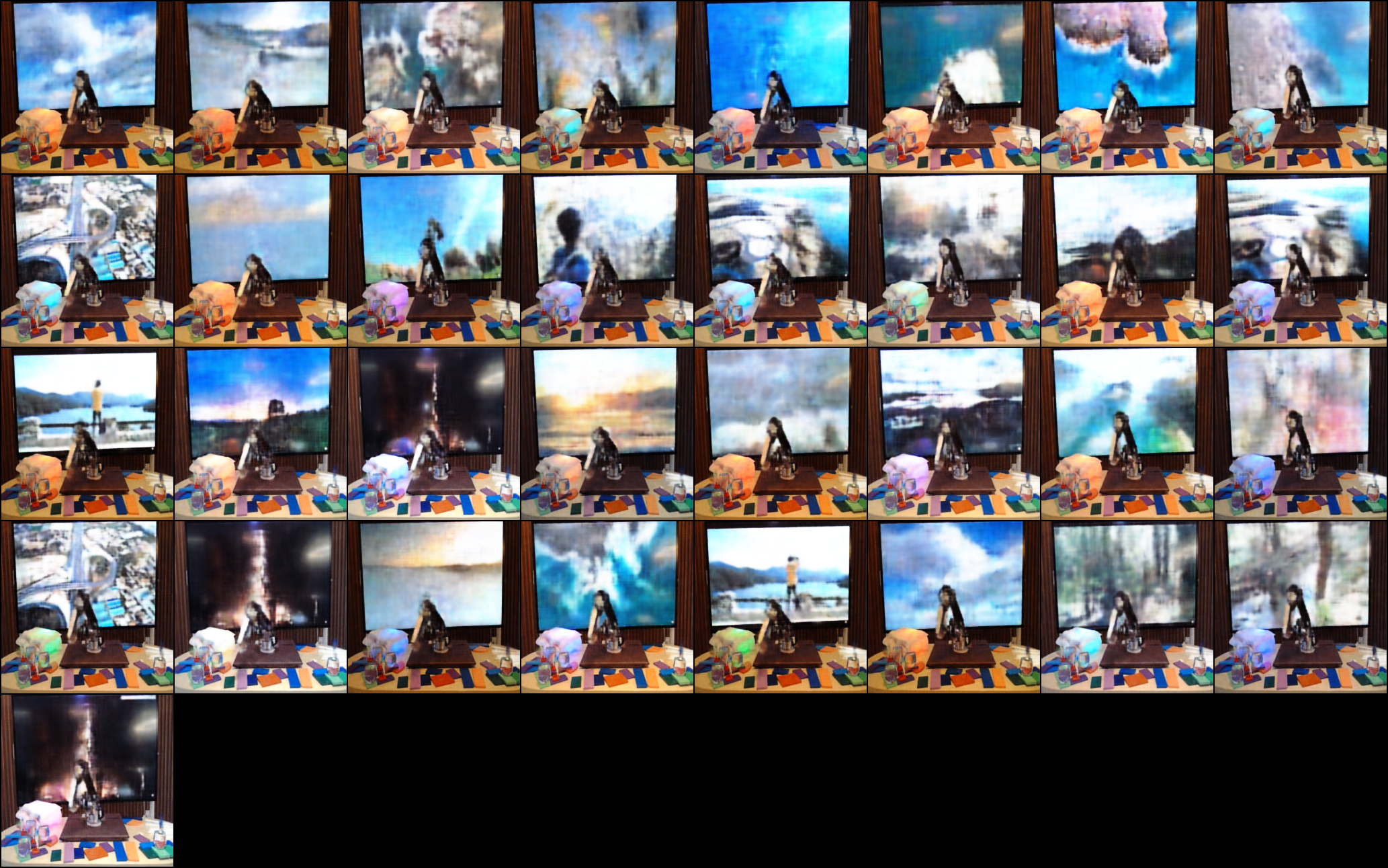}}

\caption{An autoencoder trained with 8 discrete codes to reconstruct images of a real-robotic arm with independently changing distractors (TV and lamp). The first two rows show the reconstruction of the initial images. In the middle two rows, we change one of the discrete factors to match its value in the first image, in the resulting reconstruction, the lamp has turned orange. On the bottom two rows, we change another discrete factor, to match its value from the first image, and observe that it made the robotic arm point left}
\label{fig:robot_data_compositionality}
\end{figure}

\textbf{Experiment Details:} In this task, the robot arm moves on top of a grid layout, containing $9$ different positions. We denote these as the \textit{true states}. We use two cameras to take images, for the dataset, one from the front side of the robot and the other with a top down view from above. We collect a dataset containing pixel based observations only, where the images consist of the robot arm along with the background distractors. Inspired by the exogenous noise information setup \cite{Efroni2021ppe}, we setup the robot task while there is a TV playing a video in the background, with other flashing lights nearby. The offline dataset consists of $6$ hours of robot data, with $14000$ samples from the arm, taking high level actions of move left, right, up and down. A sample point image is collected after each action, and the background distractors changes significantly, due to video and lighting in the background. The goal of the experiment is to predict accurately the ground truth state position by learning latent representations with $\modelname$.

\textbf{Experiment Results:} We evaluate the ability of $\modelname$ to accurately reconstruct the image, by learning the latent state representation while also ignoring the background distractors. This is denoted as the \textit{Image Noise}, where we compare $\modelname$ with and without VIB, alongside a baseline agent which only learns a representation. For learning latent representations, we use a multi-step inverse dynamics model \cite{lamb2022guaranteed}. In addition, we compare the ability of $\modelname$ to accurately predict the ground truth states, denoted by \textit{State Accuracy} solely from the observations, as a classification task. This is challenging since the learnt representation needs to predict ground states while ignoring the irrelevant background information. Furthermore with the learnt model, we predict the time-step for each observation as an additional metric to determine effectiveness of $\modelname$. The time-step is an indicator of the background noise that appeared in each sample; and with \textit{Temporal Noise}, we evaluate $\modelname$ to predict the time step while ignoring irrelevant information from observations. Experiment results in Figure~\ref{fig:robot_data_compositionality} shows that the use of VIB helps improve the ability of $\modelname$ to remove noise from the representation, while being able to almost perfectly predict the ground truth state of the robot.

\section{Experiment Details and Additional Results}
\label{sec:app_add_exp_results}

For all our experiments, we use existing open-source implementations of the baseline algorithms. We mostly integrate $\modelname$ with HRAC \cite{HRAC} and RIS \cite{RIS}, which are state of the art goal conditioned algorithms on the complex Ant navigation and robotic manipulation tasks. We use the same hyperparameters and default configurations as used in the baselines. For the simpler maze tasks such as from MiniGrid and KeyChest, we also use existing setups used by previous algorithms \cite{AMIGO, HRAC}. All our experiment results are based on $3$ random seeds, and we provide our implementation for reproducibility.  

\subsection{Experiment Details on Color-MNIST Dataset}
\label{app:rebuttal_details_mnist}

We include a brief description of the experiment details, used for the color-mnist example to demonstrate factorization. The pixel-based input is first passed through an encoder (a two-layer neural network) to obtain its latent representation with the dimension of 30, we then quantize the continuous representation into two groups of discrete codes, where the codebook size is 256. In the training procedure, two groups of the discrete codes are then concatenated to obtain the discretized representation, and finally passed through a decoder (another two-layer neural network), where we used reconstruction loss (MSE loss) combining with the loss for vector quantization to train the network. While in the testing procedure, we used zero vector to substitute one group of the discrete codes, and then obtain the reconstructed image by concatenating it with the other group and passing through the decoder.

\subsection{Visual MazeWorlds}
\label{sec:app_vis_maze_details}

\begin{figure}[!htbp]
    \centering
    {{\includegraphics[width=2.5cm]{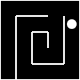} }}%
    \qquad
    {{\includegraphics[width=2.5cm]{neurips2022/figures/spiral_maze.png} }}%
    {{\includegraphics[width=4.5cm]{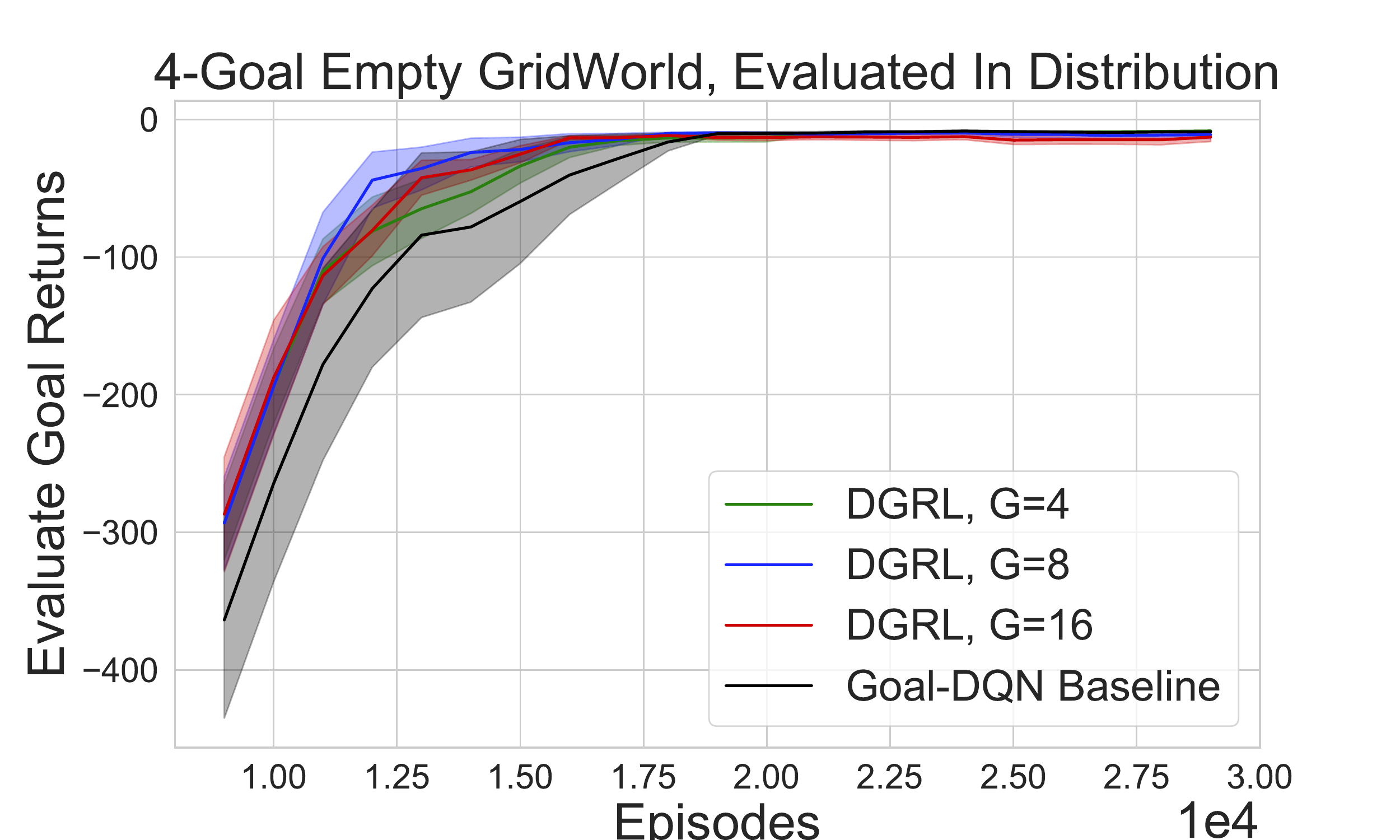}}}
    \qquad
    {{\includegraphics[width=4.5cm]{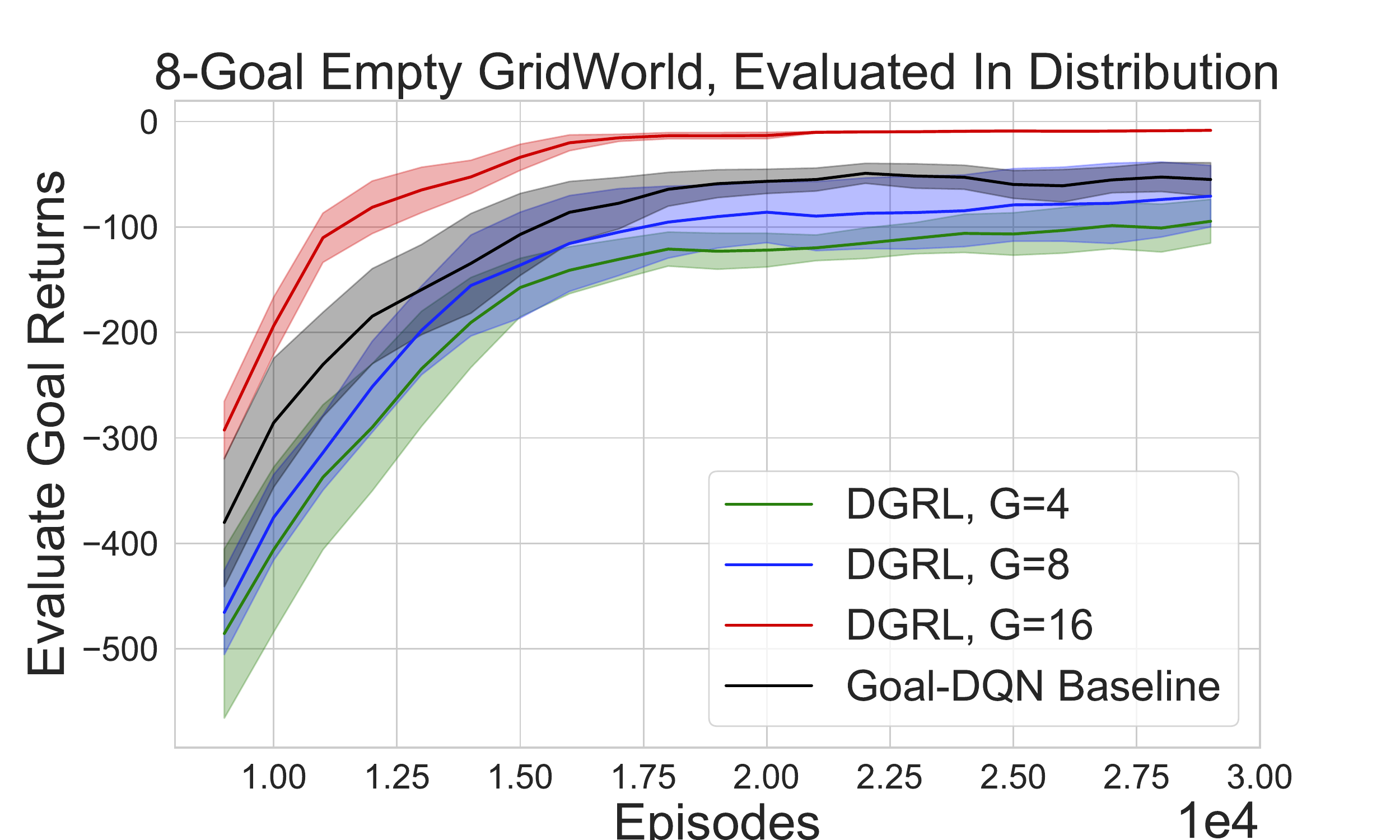} }}%
    {{\includegraphics[width=4.5cm]{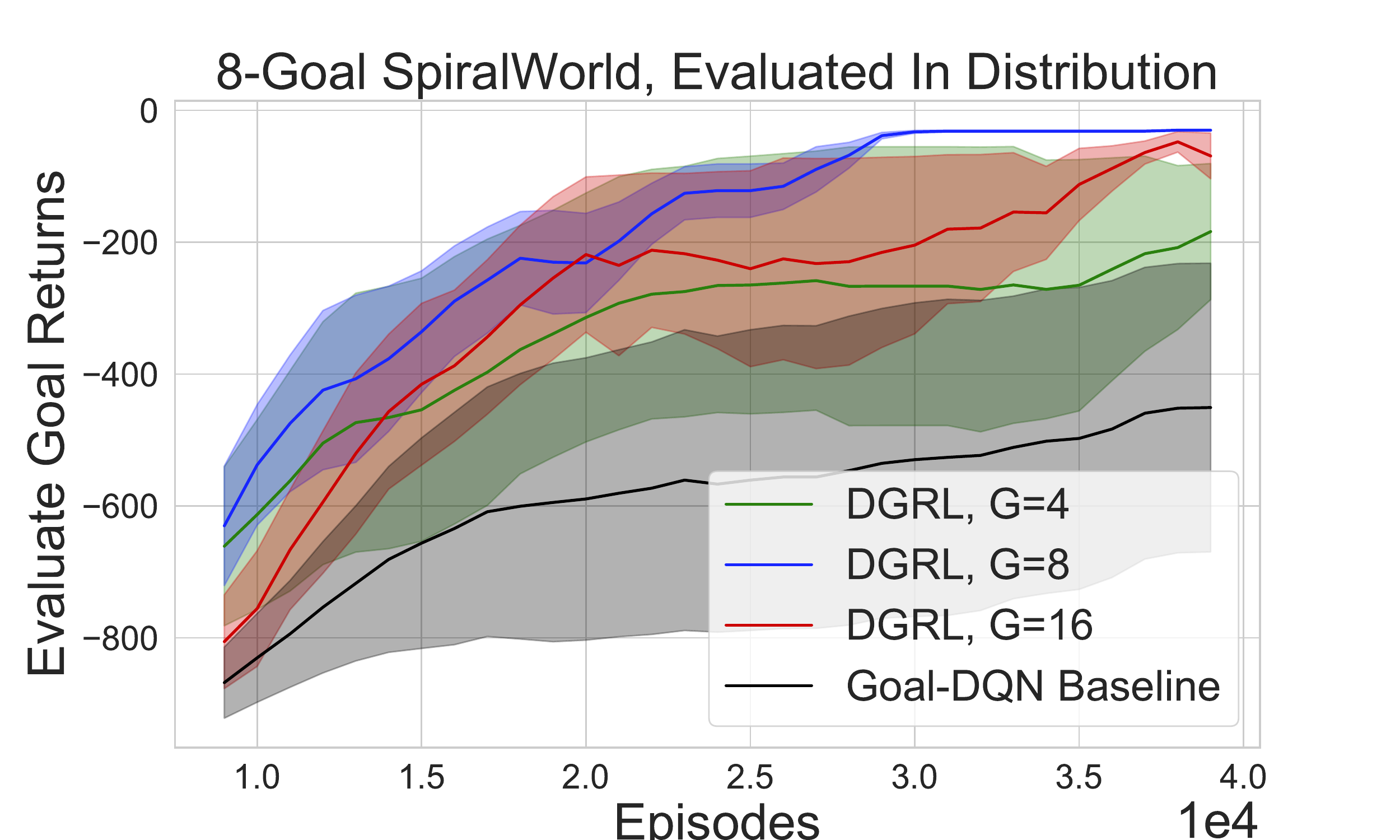} }}%

    \caption{Visualization of the SpiralWorld and LoopWorld environments. Additional experimental results in an empty 4-goal and 8-goal environment across different factors $G$.}%
    \label{fig:app_4_vis_maze}%
\end{figure}
In this section, we provide additional details for the visual maze tasks. We provide environment visualizations, for the spiralworld and loopworld environments in figure \ref{fig:app_4_vis_maze}. For this task, we use $6 \times 6$ gridworlds, where the agent receives pixel based full observations, of size $80 \times 80 \times 3$, of the environment. The agent starts deterministically in the environment from one corner of the loop, and the goals are provided in the state space, where agent receives observations of goals. The agent receives a negative reward of $-1$ at every time step, and a reward of $+5$ when reaching any of the goals. 

We train a goal conditioned Q learning agent, with epsilon greedy exploration, where the agent receives goal observations in addition to state observations. We use a simple 2 layer architecture for the value function, which additionally conditions on the goals.
In our experiments, we use 4-goals and 8-goals environment, where the agent is trained to reach each of the goals, sampled at every episode. For pixel based observations, we learn a representation $\phi$, which can be trained via any self-supervised representation learning objective. For our experiments, we use the Deep InfoMax algorithm \cite{mazoure2020deep} for training the representation learner $\phi$. Since the environment is quite simple, we found that even pre-training $\phi$ with random rollout policies is often good enough for training the value functions, as suggested by our experimental results. For $\modelname$, we additionally use a discrete bottleneck on top of the embedding, and experimentally show that $\modelname$ can significantly outperform a baseline goal DQN agent.

\subsection{KeyChest Domain}
\label{sec:app_keychest_details}
For prototype and motivation of $\modelname$, we also used a simple discrete state and action KeyChest environment, as shown in figure \ref{fig:keychest}. This task is inspired from the HRAC algorithm \cite{HRAC}, where the agent starts stochastically in the environment, and the goal is to pick up a key and open a chest. We follow a HRL setup in this task, where the higher level policy provides discrete goals in the state space for the lower level policy to reach. Since the environment has injected stochasticity, such a task requires both low-level control conditioned on the goals, provided by a higher level planner. 

The environment has a 3-dimensional state space where the first two represent the position of the agent, while the third dimension represents whether the agent has picked up the key or not. The reward function is sparse, such that the agent only receives a reward $+1$ for picking up the key, and a reward of $+5$ if it can open the chest; otherwise zero rewards elsewhere. This hints to a hard exploration task, where specification of goal  plays a key role. The lower level policy receives a goal reaching reward only, upon reaching the goal state. The higher level policy receives rewards from the environment directly. For the goal reaching reward, we use a standard Euclidean distance between states and goals, and the goal is achieved by the lower level policy if this distance is below a threshold of $0.5$.

For $\modelname$, we apply a discrete bottleneck by first learning a representation of the discrete goal in state space, using an encoder. We then apply a discrete bottleneck on the learnt embedding, followed by a decoder that maps the learnt discrete goal embedding back to the original state space. For baseline \cite{HRAC}, the algorithm uses the raw goal states provided by the higher level policy. Both the higher and lower level policies are learnt with actor-critic algorithms. The high level policy is trained based on the task reward, whereas the lower level policy is rewarded for reaching the goals, or nearby regions of the goals, provided by the higher level policy.

\subsection{MiniGrid Environments}
\label{sec:app_amigo_details}

For our experiments, we follow the setup of \cite{AMIGO} and \cite{RIDE} and evaluate $\modelname$ on a simple procedurally generated MiniGrid environment \cite{gym_minigrid}. The minigrid environments are a suite of hard exploration testbeds in RL where the task is designed such that exploration and representation of the visual observations can be disentangled. Following \cite{AMIGO}, we use the \textbf{KCharder}, KeyCorrS4R3 environment, which requires finding a key that can unlock a door which blocks a room. This door needs to be opened by the agent so as to reach the goal. In our experiments, we use the same learning rates, network architecture and other hyperparameters as in the AMIGO paper \cite{AMIGO}. 
We use the open-sourced implementation provided by the authors (for more details, see \cite{AMIGO}), and simply integrate $\modelname$ on top of the AMIGO baseline, learning a factorial representation for the goal observations.

\subsection{Ant Manipulation tasks}

We employ an ant robot with a continuous 8-dimensional action space for all three Ant manipulation tasks. Each episode terminates at 500 time steps in all three tasks.

\paragraph{AntMazeSparse} AntMazeSparse is a challenging navigation task with sparse rewards. This environment has a continuous state space including current position, velocity, target location and the current timestep $t$. The agent is provided with a sparse reward by $+1$ only if the Euclidean distance between the agent and the target location is smaller than 1, where the target location is set at $(2.0, 9.0)$ in the center corridor. 

\paragraph{AntPush} This environment has the same state space as AntMazeSparse task. The difference is that this environment has a movable block which the agent can interact with. To successfully reach the target position $(0.0, 19.0)$, the agent must push the large block to the side to clear the path to the target location. The success of the agent is defined as having the Euclidean distance of 5 from the target position. 

\paragraph{AntFall} This environment extends the navigation to three dimensions. Similar to the AntPush task, the environment still has a movable block, while the agent must move the block into a chasm instead of pushing it aside, so that it may walk over it without falling to reach the target position. The target position is fixed to $(0.0, 27.0, 4.5)$ in this environment.

For our experiments, we use the setup provided by the HRAC baseline \cite{HRAC}. In HRAC, an additional adjacency constrained is trained, along with the lower and higher level policies, such that the goals provided by higher controller are within a constrained region of the state space, that can be reachable by the lower level policy. The higher and lower level policies are both trained based on actor-critic algorithms, with separate replay buffers for each policy. The replay buffer for the higher level policy stores one every K transitions. 

We integrate $\modelname$ on top of the HRAC baseline setup, where we apply the discrete goal bottleneck based on the output embeddings from higher level policies. Compared to HRAC, $\modelname$ with HRAC would condition the lower level policies on the discrete embeddings of goals.

Figure \ref{fig:supp_hrac_ant} provides results for all the Ant manipulation tasks. Here we provide the results obtained for all the 5 different Ant environments, and we experiment for a range of discrete factors from 2 to 32. We find that while the best performing factor $G$ is not consistent for all tasks, in general factors of 4, 8 and 16 typically outperform 2 and 32 factors.

\begin{figure}[!htbp]
    \centering
    {{\includegraphics[trim=1cm 0cm 1cm 1cm, clip=true, width=5cm]{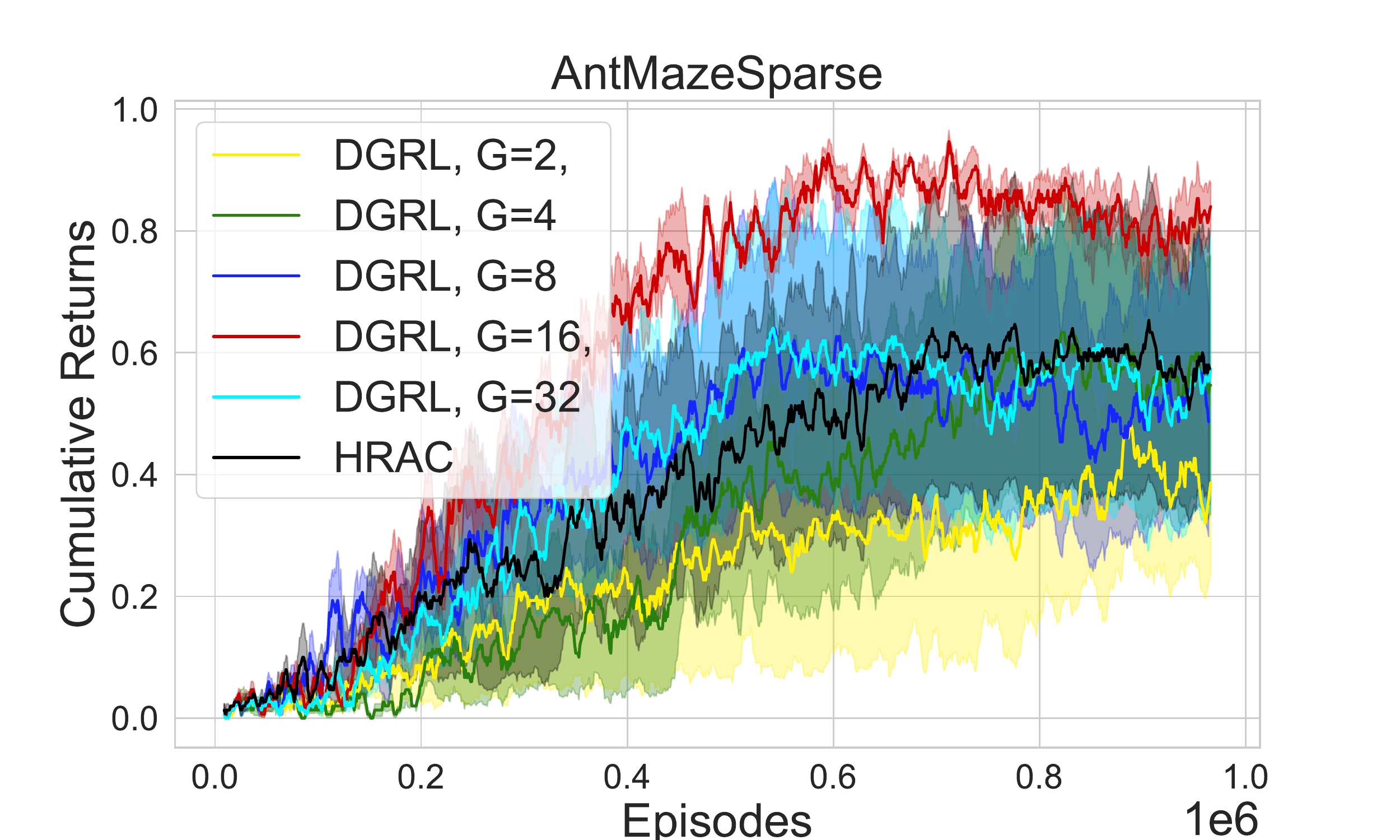} }}%
    \hspace{-0.4cm}
    {{\includegraphics[trim=1cm 0cm 1cm 1cm, clip=true,width=5cm]{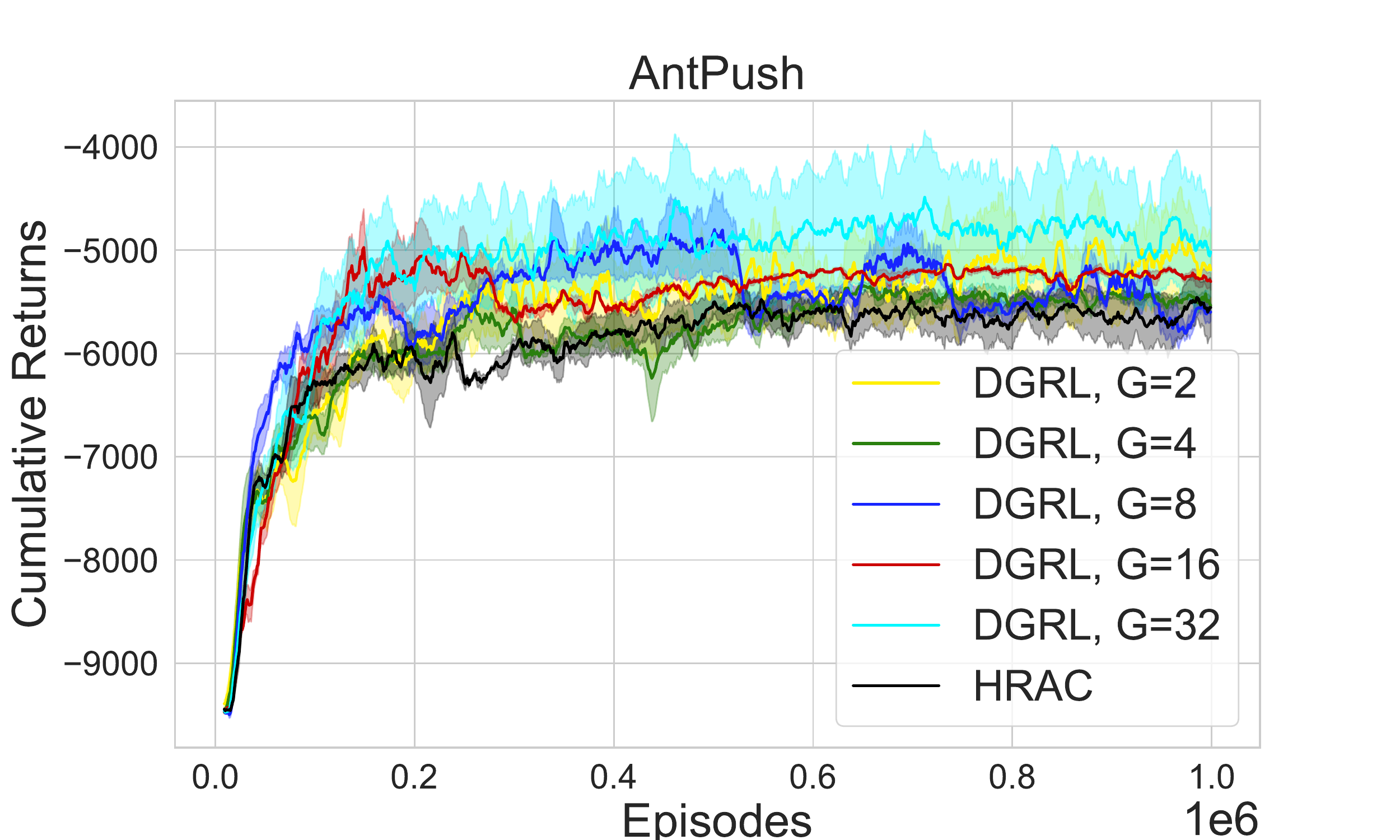} }}%
    {{\includegraphics[trim=1cm 0cm 1cm 1cm, clip=true,width=5cm]{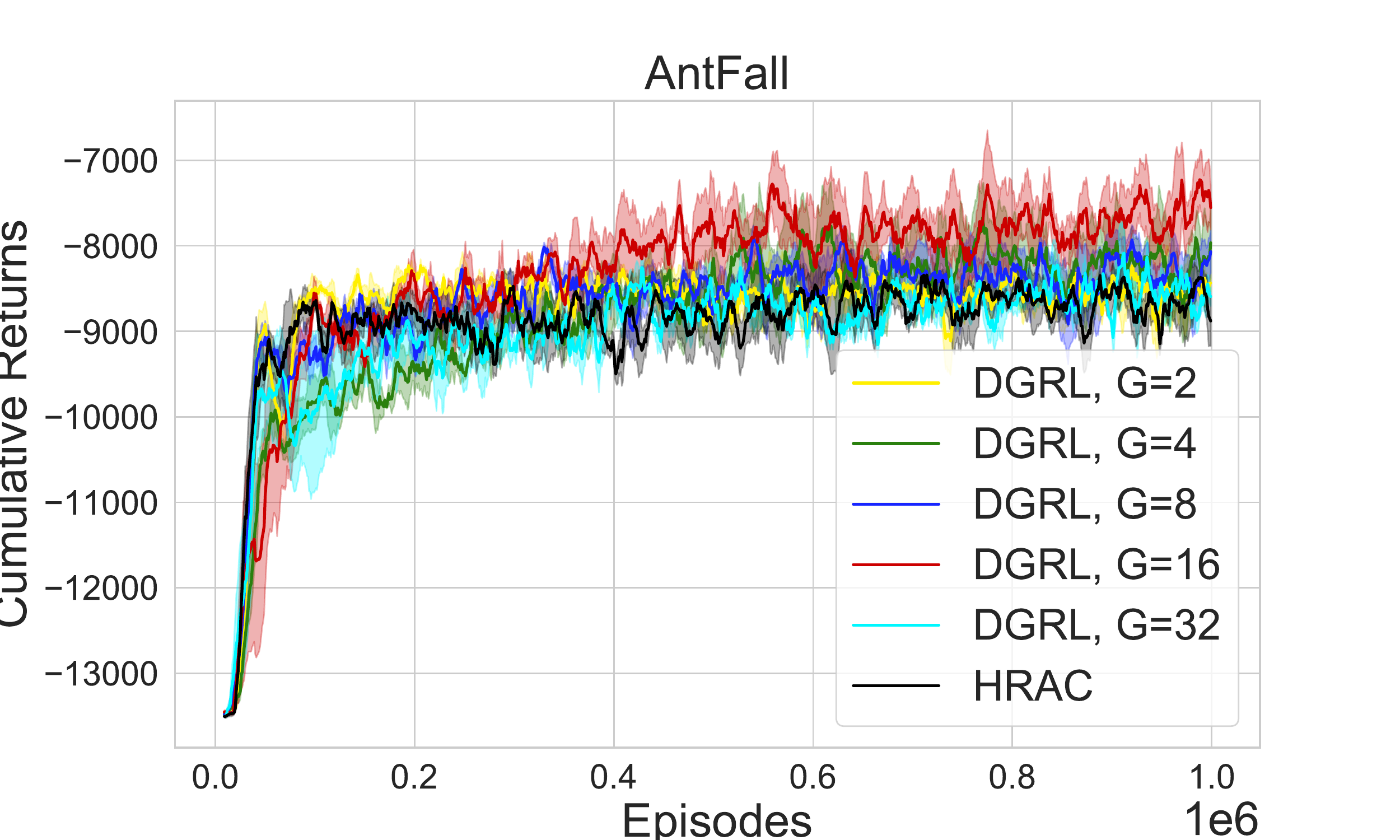} }}
    {{\includegraphics[trim=1cm 0cm 1cm 1cm, clip=true,width=5cm]{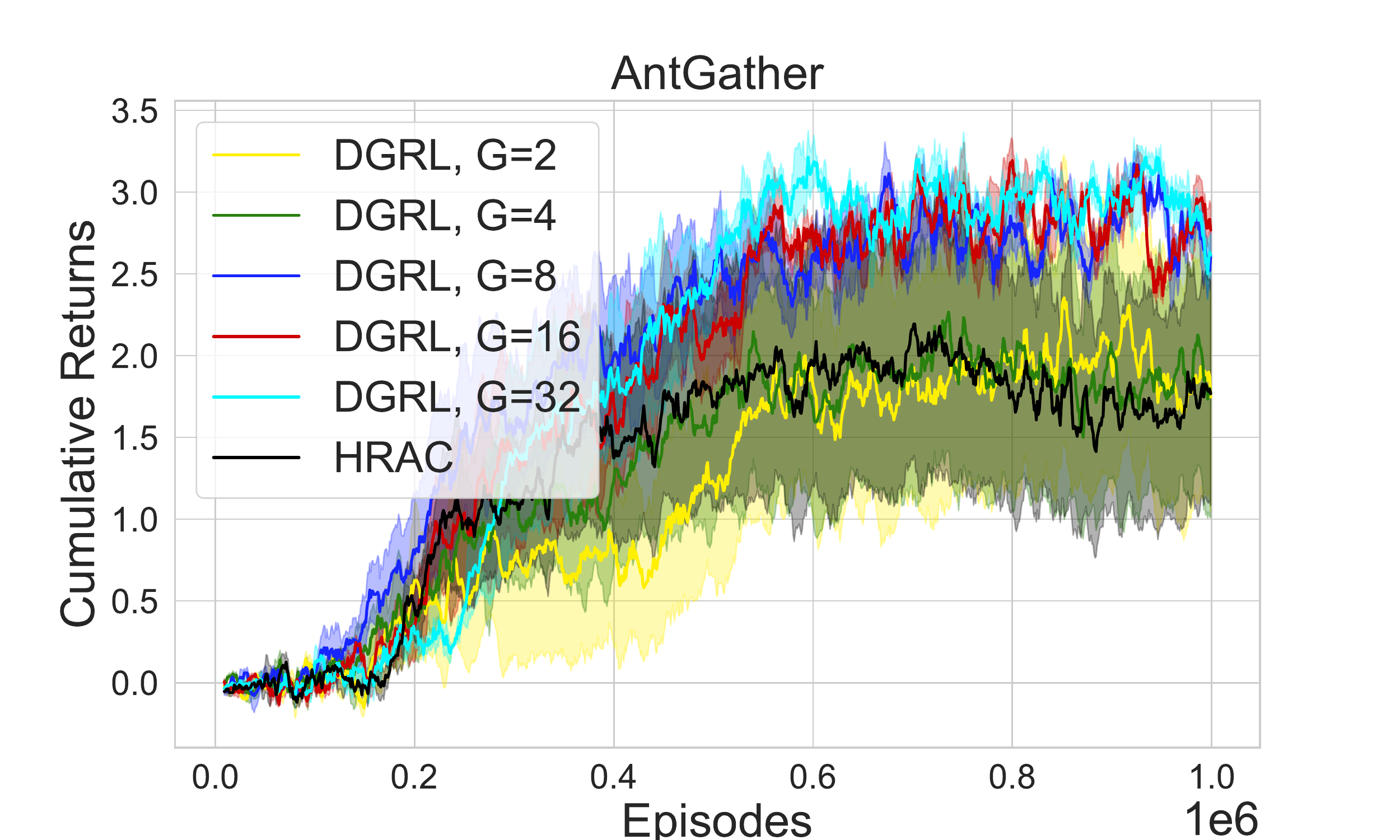} }}
    \caption{Ablation analysis on the Ant tasks. Standard goal conditioned HRL tasks, AntGather (shown on bottom left) and AntMaze, showing performance comparison of $\modelname$ with baseline HRAC \cite{HRAC} across 3 random environment seed. We include comparisons with all factors $G$ for all the Ant manipulation environments that we expeirmented on, given computation budget.}%
    \label{fig:supp_hrac_ant}%
\end{figure}

\subsection{Ant Navigation Tasks}
\label{sec:app_ant_navig_details}

Following the setup in RIS ~\cite{RIS}, we employ an ant robot with a continuous 8-dimensional action space for all four Ant navigation tasks. reasoning. Figure \ref{fig:app_AntU_result} shows an additional result on the U-shaped Ant maze navigation environment. All these environments have a continuous state space including the current position, orientation, the joint angles and the velocities. The goal is considered reached when the Euclidean distance from the target position of the environment is less than 0.5. The agent gets a $-1$ reward at each time step until the goal is reached. The $U$-shaped maze has a size of $7.5\times 18$, the $S$-shaped maze has a size of $12\times 12$, while the $\pi$-shaped maze and $\omega$-shaped maze share the same size of $16\times 16$. In the training stage, target and initial position of the agent are sampled randomly at the beginning of each episode. At evaluation, the initial state and the target position are fixed, as illustrated in Figure~\ref{fig:app_ant_navig_env}, to test the performance of the agent on challenging configurations that require temporally extended

\begin{figure}[!htbp]
\centering
\subfigure[U-shaped maze]{
{{\includegraphics[trim=0cm 0cm 0cm 0cm, clip=true, width=3cm]{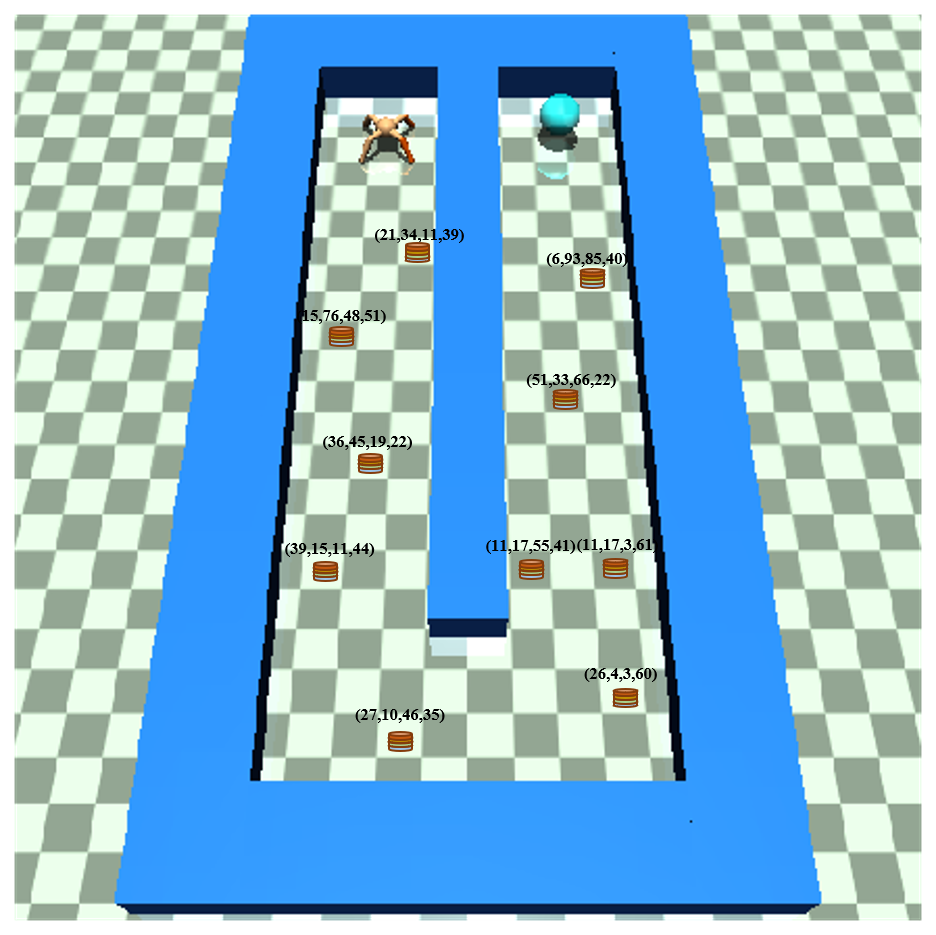} }}%
\hspace{-0.4cm}

}
\subfigure[$\Pi$-shaped maze]{
{{\includegraphics[trim=0cm 0cm 0cm 0cm, clip=true,width=3cm]{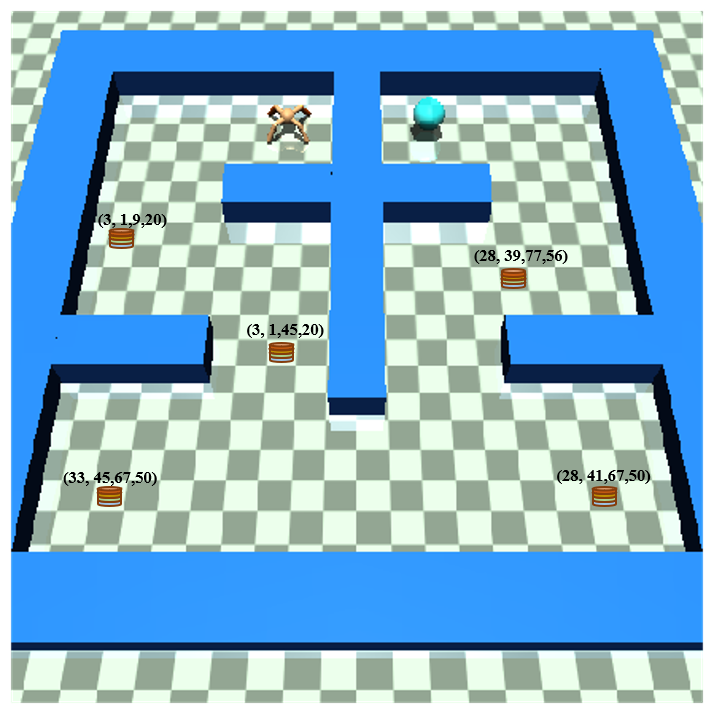} }}%

}
\subfigure[ Experiment results comparing different factors $G$ for \modelname\ integrated with RIS, and compared with a RIS baseline (baseline data provided by authors)]{
{{\includegraphics[trim=1cm 0cm 1cm 1cm, clip=true,width=5.5cm]{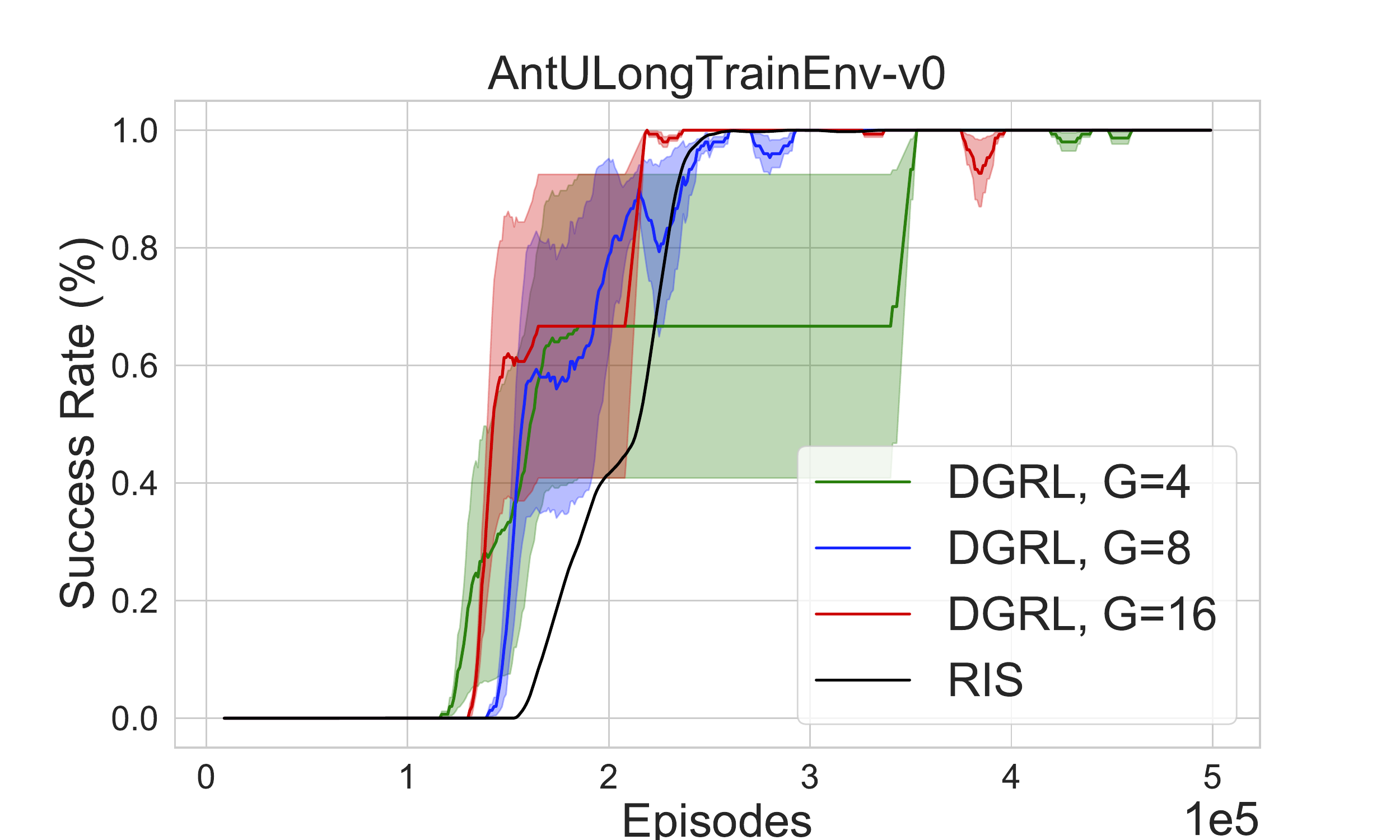} }}%
\label{fig:app_AntU_result}
}
\caption{Visualization of Ant navigation environment, which are standard goal conditioned HRL tasks. We integrate $\modelname$ on top of the existing state of the art RIS baseline \cite{RIS}. Performance comparison of success rate on AntU-shaped environment (right). We compare $\modelname$ with different factors $G$ and the RIS baseline \cite{RIS} which has already been shown to outperform other goal conditioned baselines such as LEAP \cite{LEAP} on the Ant navigation tasks. }
\label{fig:app_ant_navig_env}
\end{figure}

\subsection{Robotic Manipulation Task}
\label{sec:app_sawyer_details}

We consider the visual robot manipulation task Sawyer, from the multiworld environments of \cite{LEAP}. Our setup is based entirely on the experiment details and code provided in the open-source codebase of RIS~\cite{RIS}.
The goal of the agent is to operate a 2D position control and manipulation task. The observations of the environment are based on $84 \times 84$ RGB image of the environment. $\modelname$ is trained using a discrete bottleneck based on the learnt representation of the pixel based observations of the environment. We compare the Sawyer environment based on an existing RIS \cite{RIS} baseline which is already shown to outperform other baselines on this task. For more details on the suite of multiworld environments, including the Sawyer manipulation task, see \cite{LEAP}. In this setting, we can additionally check for generalization. At test time, the cumulative returns of the agent are provided when evaluated on a slightly different task, such as a hard configuration where we evaluate the policy and bottleneck based on learnt representations in the given task. 

\begin{wrapfigure}{r}{0.53\textwidth}
  \begin{center}
    \includegraphics[width=\textwidth]{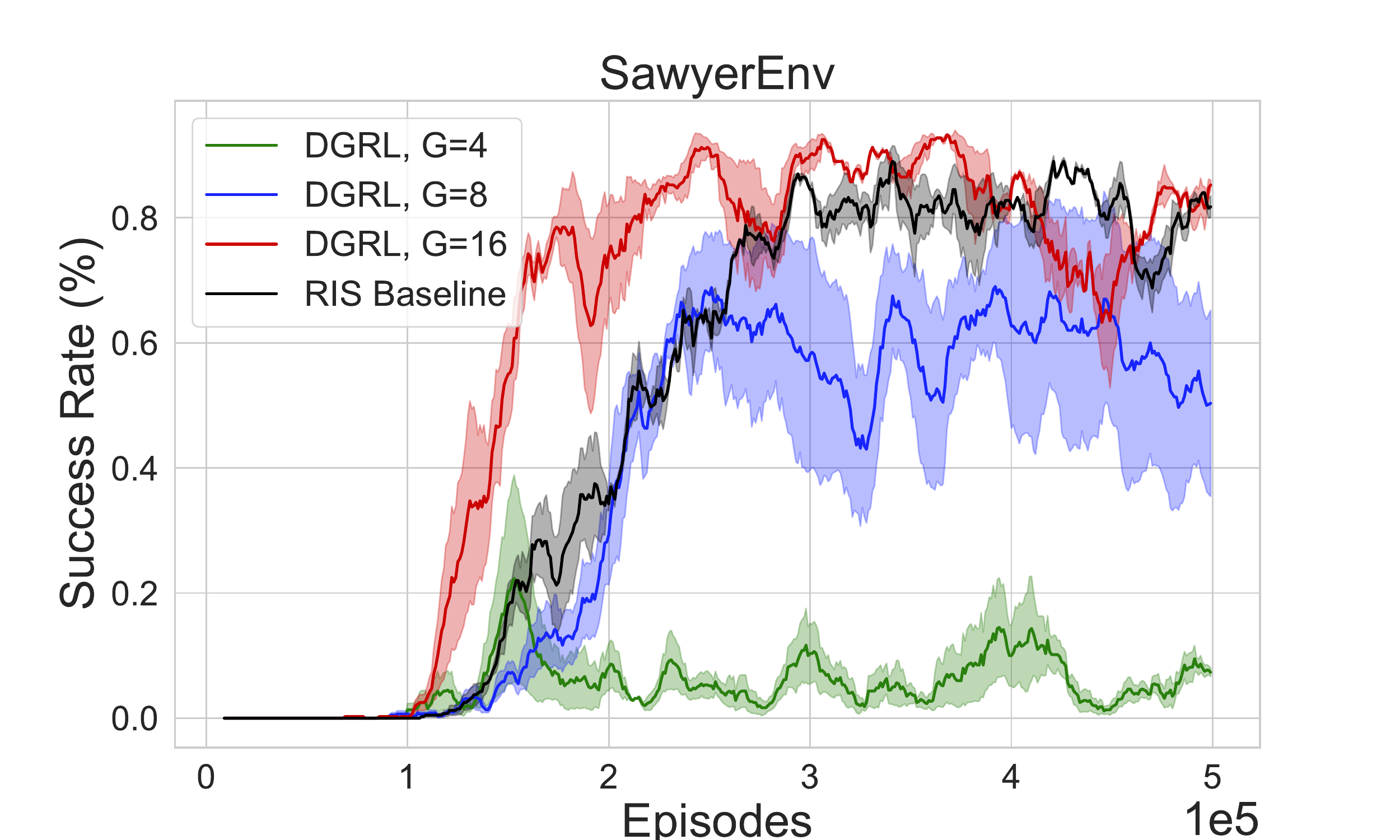}
  \end{center}
  \vspace{-1.5mm}
  \caption{Performance comparison of $\modelname$ with RIS baseline on a complex vision-based robot manipulation task}
  \label{fig:supp_sawyer_success}
  \vspace{-3mm}
\end{wrapfigure}

\section{Proof of Theorem \ref{thm:1}}
\label{sec:app_theory}

We consider a  goal-conditioned Markov decision process, defined by states $s \in \Scal$, goals $g \in \Gcal$, actions $a \in \Acal$, a reward function $r(s, a,g)$, a transition dynamics $p(s'|s, a)$,  a  maximum horizon $T$, the initial state distribution $\rho_0$, and the goal distribution $\rho_g$. The objective in goal-conditioned RL is to obtain  a policy $\pi(a|s,g)$ to maximize the expected sum of rewards $\EE_{g \sim \rho_g,((s_{t},a_{t}))_{t=1}^T }[ \sum_{t=1}^T r(s_t,a_t,g)]$ where the sequence of state-action pairs  $((s_{t},a_{t}))_{t=1}^T$ is sampled according to  $s_0\sim\rho_0$, $a_t \sim \pi(a_t | s_t,  g)$, and $s_{t+1}\sim p(s_{t+1}|s_t, a_t)$. To study the phenomenon of  the goal observations, we define 
$$
\varphi_\theta(g)=\EE_{((s_{t},a_{t}))_{t=1}^T}\left[\sum_{t=1}^T r(s_t,a_t,g) \right] ,
$$ 
where $\theta \in \RR^m$ is the vector containing model parameters learned through goals observed during training phase, $g_1,\dots,g_n$. Define $\one\{a=b\}=1$ if $a=b$ and $\one\{a=b\}=0$ if $a\neq b$. Moreover, $\varphi_\theta \circ \varsigma$ represents the composition of functions $\varphi_\theta$ and $\varsigma$.

In our proof, we will use the following previous result:

\begin{lemma}[Bretagnolle-Huber-Carol inequality] \citep[Proposition A.6.6]{van1996} \label{lemma:1}
        If $X_1, \dots, X_K$ are multinomially distributed with parameters $m$ and $p_1, \dots, p_K$, then for $\bar M > 0$,
        \[
        \PP\left( \sum_{k=1}^K \left|\frac{X_k}{m} - p_k \right| \geq \bar M \right) \leq 2^K \exp\left( -\frac{m\bar M^2}{2} \right).
        \]
\end{lemma}  

\begin{proof}
We decompose the goal probability $p(g)$ into $p(g)=\sum_k p(g|g\in \Gcal_k)p(g \in \Gcal_k)$ with the neighborhood set $\{\Gcal_k\}_k$ and  analyze the concentrations of random variables in terms of  $p(g|g\in \Gcal_k)$ and $p(g \in \Gcal_k)$. Let  $\varsigma \in \{\id, q\}$. We define $M=\sup_{g \in \Gcal}\varphi_\theta(g) $. Then, \begin{align*}
\EE_{g \sim \rho_g}[ (\varphi_\theta \circ \varsigma)(g))]  =\sum_{k}  \EE_{g \sim \rho_g}[(\varphi_\theta \circ \varsigma)(g)|g\in  \Gcal_{k}]\Pr(g_{}\in\Gcal_{k}^{}) ,
\end{align*}
Using this, we  decompose the difference as 
\begin{align} \label{eq:1}
& \frac{1}{n} \sum_{i=1}^n  (\varphi_\theta \circ \varsigma)(g_{i})-\EE_{g \sim \rho_g}[ (\varphi_\theta \circ \varsigma)(g))]  
\\ \nonumber & =\sum_{k} \EE_{g \sim \rho_g}[(\varphi_\theta \circ \varsigma)(g)|g\in  \Gcal_{k}]\left( \frac{|\Ical_{k}^{}|}{n}-\Pr(g\in  \Gcal_{k}^{})\right)
\\ \nonumber & \quad +\left( \frac{1}{n} \sum_{i=1}^n (\varphi_\theta \circ \varsigma)(g_{i})-\sum_{k}^{} \EE_{g \sim \rho_g}[(\varphi_\theta \circ \varsigma)(g)|g\in  \Gcal_{k}]\frac{|\Ical_{k}^{}|}{n}\right). 
\end{align}
Since 
$
\frac{1}{n} \sum_{i=1}^n (\varphi_\theta \circ \varsigma)(g_{i})=\frac{1}{n}\sum_{k}^{}  \sum_{i \in \Ical_{k}}(\varphi_\theta \circ \varsigma)(g_{i})$, 
\begin{align*}
&  \frac{1}{n} \sum_{i=1}^n (\varphi_\theta \circ \varsigma)(g_{i}
)-\sum_{k}^{} \EE_{g \sim \rho_g}[(\varphi_\theta \circ \varsigma)(g)|g\in  \Gcal_{k}]\frac{|\Ical_{k}^{}|}{n}\\ & =\frac{1}{n}\sum_{k}^{}|\Ical_{k}^{}|\left(\frac{1}{|\Ical_{k}^{}|}\sum_{i \in \Ical_{k}}(\varphi_\theta \circ \varsigma)(g_{i})-\EE_{g \sim \rho_g}[(\varphi_\theta \circ \varsigma)(g)|g\in  \Gcal_{k}] \right).
\
\end{align*}
Substituting these into equation \eqref{eq:1} yields
\begin{align} \label{eq:2} 
& \frac{1}{n} \sum_{i=1}^n  (\varphi_\theta \circ \varsigma)(g_{i})-\EE_{g \sim \rho_g}[ (\varphi_\theta \circ \varsigma)(g))]   
\\ \nonumber & =\sum_{k}^{} \EE_{g \sim \rho_g}[(\varphi_\theta \circ \varsigma)(g)|g\in  \Gcal_{k}]\left( \frac{|\Ical_{k}^{}|}{n} - \Pr(g\in    \Gcal_{k}^{})\right)
 \\ \nonumber & \quad +\frac{1}{n}\sum_{k}^{}|\Ical_{k}^{}|\left(\frac{1}{|\Ical_{k}^{}|}\sum_{i \in \Ical_{k}}(\varphi_\theta \circ \varsigma)(g_{i})-\EE_{g \sim \rho_g}[(\varphi_\theta \circ \varsigma)(g)|g\in  \Gcal_{k}] \right)
\\ \nonumber & \le M\sum_{k}^{}\left| \frac{|\Ical_{k}^{}|}{n}-\Pr(z\in  \Gcal_{k}^{})\right| +\frac{1}{n}\sum_{k}^{}|\Ical_{k}^{}|\left(\frac{1}{|\Ical_{k}^{}|}\sum_{i \in \Ical_{k}}(\varphi_\theta \circ \varsigma)(g_{i})-\EE_{g \sim \rho_g}[(\varphi_\theta \circ \varsigma)(g)|g\in  \Gcal_{k}] \right)
\end{align}
By using Lemma \ref{lemma:1} by setting $\delta= 2^K \exp\left( -\frac{m \bar M^2}{2} \right)$ and solving for $M$, we have that for any $\delta>0$, with probability at least $1-\delta$,
$\sum_{k}^{}\left| \frac{|\Ical_{k}^{}|}{n}-\Pr(z\in  \Gcal_{k}^{})\right| \le \sqrt{\frac{2\ln(2^{|\Qcal|}/\delta)}{n}}\le \sqrt{\frac{2|\Qcal|\ln(2/\delta)}{n}}$, where the last inequality follows from the fact that $1/\delta \le 1/\delta^{|Q|}$ as $\delta \in (0,1)$ and $|Q| \ge 1$.  Here, notice that the term of $\sum_{k}^{}\left| \frac{|\Ical_{k}^{}|}{n}-\Pr(z\in  \Gcal_{k}^{y})\right|$ does not depend on $\theta$. Moreover, note that for any $(f,h,M)$ such that $M>0$ and $B\ge0$ for all $X$, we have that 
$
\PP(f(X)\ge M)\ge \PP(f(X)>M) \ge  \PP(Bf(X)+h(X)>BM+h(X)),
$
where the probability is with respect to the randomness of  $X$.
Thus,  by combining this and equation \eqref{eq:2}, we have that for any $\delta>0$, with probability at least $1-\delta$,
 the following holds for all $\theta$,
\begin{align} \label{eq:4}
&\frac{1}{n} \sum_{i=1}^n  (\varphi_\theta \circ \varsigma)(g_{i})-\EE_{g \sim \rho_g}[ (\varphi_\theta \circ \varsigma)(g))] 
\\ \nonumber & \le \frac{1}{n}\sum_{k=1}^{|\Qcal|}|\Ical_{k}^{}|\left(\frac{1}{|\Ical_{k}^{}|}\sum_{i \in \Ical_{k}}(\varphi_\theta \circ \varsigma)(g_{i})-\EE_{g \sim \rho_g}[(\varphi_\theta \circ \varsigma)(g)|g\in  \Gcal_{k}] \right)+c  \sqrt{\frac{2\ln(2/\delta)}{n}}
\\ \nonumber & =\frac{1}{n}\sum_{k\in \Ical_{\Qcal}}|\Ical_{k}^{}|\left(\frac{1}{|\Ical_{k}^{}|}\sum_{i \in \Ical_{k}}(\varphi_\theta \circ \varsigma)(g_{i})-\EE_{g \sim \rho_g}[(\varphi_\theta \circ \varsigma)(g)|g\in  \Gcal_{k}] \right)
+c  \sqrt{\frac{2\ln(2/\delta)}{n}}
\end{align}
If $\varsigma=\id$, then 
\begin{align*}
&\frac{1}{n}\sum_{k\in \Ical_{\Qcal}}|\Ical_{k}^{}|\left(\frac{1}{|\Ical_{k}^{}|}\sum_{i \in \Ical_{k}}(\varphi_\theta \circ \varsigma)(g_{i})-\EE_{g \sim \rho_g}[(\varphi_\theta \circ \varsigma)(g)|g\in  \Gcal_{k}] \right) 
\\ & = \frac{1}{n}\sum_{k\in \Ical_{\Qcal}}|\Ical_{k}^{}|\left(\frac{1}{|\Ical_{k}^{}|}\sum_{i \in \Ical_{k}}\varphi_\theta (g_{i})-\EE_{g \sim \rho_g}[\varphi_\theta (g)|g\in  \Gcal_{k}] \right) = \omega(\theta).
\end{align*}
If $\varsigma=q$, then 
\begin{align*}
&\frac{1}{n}\sum_{k\in \Ical_{\Qcal}}|\Ical_{k}^{}|\left(\frac{1}{|\Ical_{k}^{}|}\sum_{i \in \Ical_{k}}(\varphi_\theta \circ \varsigma)(g_{i})-\EE_{g \sim \rho_g}[(\varphi_\theta \circ \varsigma)(g)|g\in  \Gcal_{k}] \right)
\\ & =\frac{1}{n}\sum_{k\in \Ical_{\Qcal}}|\Ical_{k}^{}|\left(\frac{1}{|\Ical_{k}^{}|}\sum_{i \in \Ical_{k}}(\varphi_\theta \circ q)(g_{i})-\EE_{g \sim \rho_g}[(\varphi_\theta \circ q)(g)|g\in  \Gcal_{k}] \right)
\\ & =\frac{1}{n}\sum_{k\in \Ical_{\Qcal}}|\Ical_{k}^{}|\left(\frac{1}{|\Ical_{k}^{}|}\sum_{i \in \Ical_{k}}\varphi_\theta (\Qcal_k)-\varphi_\theta (\Qcal_k) \right) =0
\end{align*}
Therefore, for any $\delta>0$,  with  probability at least $1-\delta$, the following holds for any $\theta \in \RR^m$ and $\varsigma \in \{\id, q\}$: \begin{align*}
\EE_{g \sim \rho_g}[ (\varphi_\theta \circ \varsigma)(g))]  \ge \frac{1}{n} \sum_{i=1}^n  (\varphi_\theta \circ \varsigma)(g_{i})-  c  \sqrt{\frac{2\ln(2/\delta)}{n}} - \one\{\varsigma=\id\}\omega(\theta).
\end{align*} 
Define $u_{n}(S,\theta)=\max_{k \in \Ical_\Qcal}\frac{1}{|\Ical_{k}^{}|}\sum_{i \in \Ical_{k}}\varphi_\theta (g_{i})-\EE_{g \sim \rho_g}[\varphi_\theta (g)|g\in  \Gcal_{k}]$. Then, since $\sum_{k\in \Ical_{\Qcal}}|\Ical_{k}|=n$,
\begin{align*}
\omega(\theta) \le u_{n}(S,\theta)\frac{1}{n}\sum_{k\in \Ical_{\Qcal}}|\Ical_{k}|=u_{n}(S,\theta). 
\end{align*}
For each $k \in [|\Qcal|]$, if $p(g \in \Gcal_k) = 0$, then the probability of the event of $|\Ical_k|\ge1$ is zero. Thus, by taking union bounds, with probability one, for all $k \in \Ical_{\Qcal}$, $|\Ical_k|\rightarrow \infty$ as $n \rightarrow \infty$. Therefore, by using the uniform law of large numbers (Theorem 2 of \citep{jennrich1969asymptotic}) and union bounds over  $k \in \Ical_{\Qcal}$ (noticing that $|\Ical_{\Qcal}|\le |\Qcal|$ is finite), we have that $\sup_{\theta \in \Theta} |u_{n}(S,\theta)| {\xrightarrow  {P}}\ 0$ when $n \rightarrow \infty$. Thus, 
$$
0\le \sup_{\theta \in \Theta}\left|\omega(\theta) \right| \le\sup_{\theta \in \Theta} |u_{n}(S,\theta)| {\xrightarrow  {P}}\ 0 \ \ \ \text{ when } \ \ \ n \rightarrow \infty.   
$$

\end{proof}

\section{Algorithm in Goal Conditioned RL}
\label{sec:app_algo}
We present the entire algorithm of DGRL built on top of RIS in algoirithm \ref{algo:algo}.
\begin{algorithm}[t]
    \caption{RIS with DGRL (changes to RIS in blue)}
   \label{algo:algo}
\begin{algorithmic}[1]
  \STATE Initialize replay buffer $D$
  \STATE Initialize $Q_\phi$, $\pi_\theta$, $\pi^H_\psi$
  \FOR {k = 1, 2, ...}
  \STATE Collect experience in $D$ using $\pi_\theta$ in the environment
  \STATE Sample batch $(s_t, a_t, r_t, s_{t+1}, g) \sim D$ with HER
  \STATE Sample batch of subgoal candidates $z_e \sim D$
  \STATE Update $Q_\phi$ using  Policy Evaluation
  \STATE Update $\pi^H_\psi$ using  High-Level Policy Improvement
  \STATE Output subgoal $z_e$ using $z_e\sim\pi_{\psi_{k+1}}^{H}(\cdot|s,g)$
  \STATE \textcolor{blue}{Output discrete goal embedding $z_q$ using Eq. \ref{discrete_codes} \textit{ (Discretization module)}}
  \STATE Compute prior policy with \textcolor{blue}{discrete} sub-goal embeddings 
  \begin{equation}
      \pi_{k}^{\text {prior }}(a \mid s, g):=\mathbb{E}_{z_{e} \sim \pi^{H}(. \mid s, g)}\left[\pi_{\theta_{k}^{\prime}}\left(a \mid s, \textcolor{blue}{z_{q}}\right)\right]
  \end{equation}
  \STATE Update $\pi_\theta$ using  Policy Improvement with Imagined Subgoals (Eq.~9 in RIS~\cite{RIS})
  \STATE \textcolor{blue}{Update discretization module using $\mathcal{L}_{\mathrm{discretization}} = \frac{\beta}{G}  \sum^{G}_{i}||c_{i}-\sg(e_{o_i})||^2_2$ }
  \ENDFOR
\end{algorithmic}
\end{algorithm}

\clearpage
\section{Code Snippet of $\modelname$}
\label{sec:app_code}

We show a simple code snippet of algorithm \ref{algo:algo} below. 

\begin{lstlisting}[language=Python]
class RIS(object):
    def __init__(self):
        discrete_cfg = {'groups': self.args.groups, 'n_embed': self.args.n_embed}
        self.vq_layer = VectorQuantizerEMA(state_dim, discrete_cfg['n_embed'], discrete_cfg['groups']).to(device)
        params = list(self.critic.parameters()) + list(self.vq_layer.parameters())
        self.critic_optimizer = torch.optim.Adam(params, lr=q_lr)

    def select_action(self, state, goal):
        with torch.no_grad():
            state = self.encoder(state)
            goal = self.encoder(goal)
            goal_embed, _, goal_ind = self.vq_layer(goal)
            action, _, _ = self.actor.sample(state, goal_embed)
        return action
        
    def train(self, state, action, reward, next_state, done, goal, subgoal):
        if self.image_env: ## FOR SAWYER ENVS
            state_z = self.encoder(state)
            next_state_z = self.encoder(next_state)
            goal_z = self.encoder(goal)
            subgoal_z = self.mlp_encoder(subgoal)
            encoder_loss = self.driml_loss(state, next_state, action)
            if self.args.use_vq == 'true':
                _, vq_loss_state, _ = self.vq_layer(state_z)
                _, vq_loss_next_state, _ = self.vq_layer(next_state_z)
                _, vq_loss_goal, _ = self.vq_layer(goal_z)
                _, vq_loss_subgoal, _ = self.vq_layer(subgoal_z)
                rep_loss = vq_loss_state + vq_loss_next_state + vq_loss_goal + vq_loss_subgoal
            rep_loss += encoder_loss

        else: ## for non-image or non-pixel based envs - ANT ENVS
            state_z = self.mlp_encoder(state)
            next_state_z = self.mlp_encoder(next_state)
            goal_z = self.mlp_encoder(goal)
            subgoal_z = self.mlp_encoder(subgoal)
            encoder_loss = self.autoencoder_loss(state)
            if self.args.use_vq == 'true':
                _, vq_loss_state, _ = self.vq_layer(state_z)
                _, vq_loss_next_state, _ = self.vq_layer(next_state_z)
                _, vq_loss_goal, _ = self.vq_layer(goal_z)
                _, vq_loss_subgoal, _ = self.vq_layer(subgoal_z)
                rep_loss = vq_loss_state + vq_loss_next_state + vq_loss_goal + vq_loss_subgoal
            rep_loss +=encoder_loss



\end{lstlisting}

\end{document}